\documentclass[lettersize,journal]{IEEEtran}
\usepackage{amsmath,amsfonts}
\usepackage{algorithmic}
\usepackage{algorithm}
\usepackage{array}
\usepackage[caption=false,font=normalsize,labelfont=sf,textfont=sf]{subfig}
\usepackage{textcomp}
\usepackage{stfloats}
\usepackage{url}
\usepackage{verbatim}
\usepackage{graphicx}
\usepackage{cite}
\usepackage{booktabs} 

\usepackage{amsthm}
\theoremstyle{definition}
\newtheorem{theorem}{Theorem}
 
\newtheorem{proposition}{Proposition}
\newtheorem{lemma}{Lemma}

\begin{document}

\title{Drift-Aware Federated Learning: A Causal Perspective}

\author{Yunjie Fang, Sheng Wu, Tao Yang,~\IEEEmembership{Member,~IEEE,} Xiaofeng Wu and Bo Hu,~\IEEEmembership{Member,~IEEE,}
\thanks{Y. Fang, S. Wu, T. Yang (Corresponding author), X. Wu and B. Hu are with the Department of Electronic Engineering, School of Information Science and Technology, Fudan University, Shanghai 200438, China (e-mail: \{fangyj21, swu22\}@m.fudan.edu.cn, \{taoyang, xiaofeng wu, hfeng, bohu\}@fudan.edu.cn). H. Feng and B. Hu are also with the the Shanghai Institute of Intelligent Electronics and Systems, Shanghai 200433, China.}
}

\maketitle

\begin{abstract}
Federated learning (FL) facilitates collaborative model training among multiple clients while preserving data privacy, often resulting in enhanced performance compared to models trained by individual clients. However, factors such as communication frequency and data distribution can contribute to feature drift, hindering the attainment of optimal training performance. This paper examine the relationship between model update drift and global as well as  local optimizer from causal perspective. The influence of the global optimizer on feature drift primarily arises from the participation frequency of certain clients in server updates, whereas the effect of the local optimizer is typically associated with imbalanced data distributions.To mitigate this drift, we propose a novel framework termed Causal drift-Aware Federated lEarning (CAFE). CAFE exploits the causal relationship between feature-invariant components and classification outcomes to independently calibrate local client sample features and classifiers during the training phase. In the inference phase, it eliminated the drifts in the global model that favor frequently communicating clients.Experimental results demonstrate that CAFE's integration of feature calibration, parameter calibration, and historical information effectively reduces both drift towards majority classes and tendencies toward frequently communicating nodes.
\end{abstract}


\section{Introduction}
\IEEEPARstart{F}{ederated} learning (FL) has emerged as a novel paradigm within distributed computing, attracting significant attention due to its exceptional capacity to leverage data residing on numerous edge devices while strictly adhering to stringent data privacy and security protocols \cite{ref16}. Despite its numerous advantages, FL faces substantial challenges that hinder its widespread adoption and optimal operational efficiency. One of the most significant challenges is feature drift, particularly in environments characterized by non-independent and identically distributed (non-iid) data. feature drift can be attributed to communication frequency imbalance and class imbalance. These two forms of imbalance can lead to unfair treatment of underrepresented classes and clients with limited participation. Class imbalance subtly drifts the optimizer's trajectory towards the majority class, resulting in drift model predictions. Participation imbalance exacerbates this issue, as it is induced by the inherent heterogeneity in computing and communication capabilities among participating nodes. The aggregation process tends to favor nodes that frequently contribute their locally optimized models to the central server.

As the class imbalance is concerned, when the majority class in the dataset far exceeds other classes, the model overfits to the majority class, thereby reducing its ability to recognize minority classes. This leads to suboptimal model performance and poor generalization in complex real-world tasks such as Internet of Things (IoT) applications or medical image recognition \cite{ref27,ref23}. A common solution is to introduce class prior knowledge at the server side, which can effectively mitigate the negative impact of class imbalance on FL models. However, this strategy has a notable limitation: class prior knowledge is often unavailable in many real-world scenarios. Meanwhile, existing research on participant imbalance primarily focuses on convergence or fairness issues under imbalanced conditions, while the mechanisms by which participant imbalance affects federated learning network performance remain underexplored.

Multiple clients scattered across different locations independently collect data from non-iid distributions. They initialize and train local models based on a shared global model to obtain local optimal solutions. Limited by participant imbalance, some frequently communicating clients may dominate the direction of model updates, causing the global model to overly rely on their update directions \cite{ref28,ref15}. A common issue in mobile edge computing, clients in poor channel conditions may never be selected \cite{ref33}. Typical solutions include controlling client selection probability \cite{ref34,ref35,ref36} or dynamically adjusting the federated learning framework or training procedures based on client capabilities \cite{ref37,ref38,ref39}. While these approaches mitigate feature drift to some extent, controlling client selection probabilities can result in passive waiting for some clients, leading to outdated gradients. On the other hand, dynamical adjustment methods may underutilize data from clients with limited computational power.

To reveal the fundamental causes of feature drift, this paper first theoretically analyzes the complex mechanisms underlying the drifts in the traditional FL framework and identifies the inherent attributes of classifier feature drifts. We observe that feature drift arises from both the local optimizer and the global optimizer. Specifically, the drift produced by the local optimizer stems from the imbalance in local data distribution, while the drift produced by the global optimizer results from differences in client participation frequency during global aggregation.
\IEEEpubidadjcol

To address these challenges, this paper proposes a novel framework, causal drift-aware federated learning (CAFE), which provides a comprehensive and essential calibration framework for feature drift. We construct a causal graph to depict the causal relationship between sample features and federated classification results. During the training phase, CAFE utilizes the causal relationship between feature invariance components and classification results to calibrate local sample features and classifiers of clients. In the inference phase, it corrects the drift of the global model which has been towards frequently communication clients. 

\begin{figure*}
    \centering
    \includegraphics[width=0.7\linewidth]{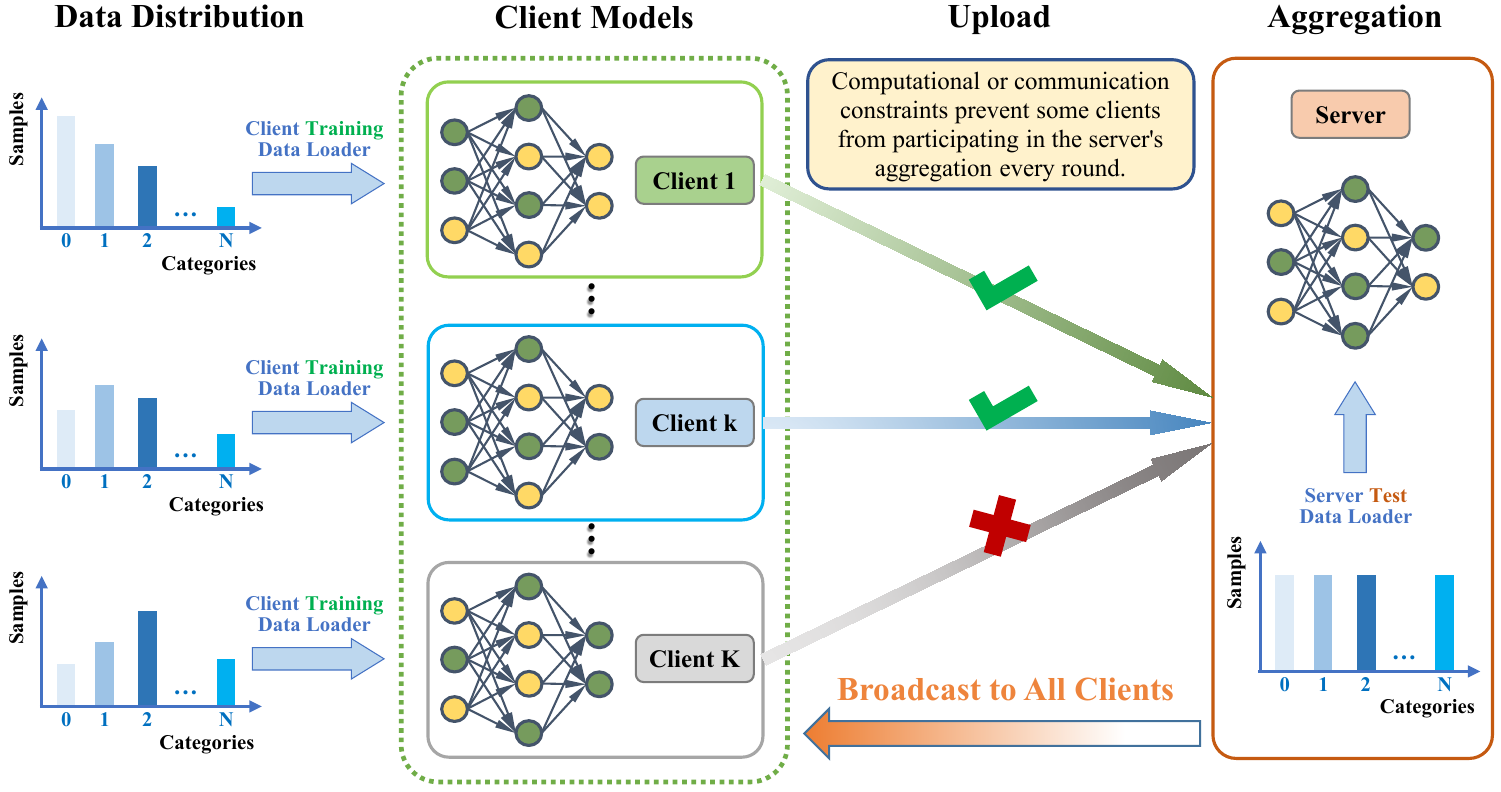}
    \caption{The reasons for feature drift in federated learning. The non-iid data distributions among clients cause their models to converge to different optimal points. Clients trained on class-imbalanced data exhibit low recognition accuracy for minority classes, while participation imbalance results in the global model being biased towards clients that communicate more frequently.}
    \label{fig:enter-label}
\end{figure*}

To the best of our knowledge, this work is the first to highlight the critical role of class sample size and client communication frequency in achieving drift-aware federated learning. Experimental results demonstrate that the proposed method cultivates a drift-aware model resilient to the detrimental effects of class imbalance and participation imbalance.

Specifically, our contribution can be expired into three points as follows:

\begin{itemize}
\item Our study elucidates the impact of device heterogeneity (e.g., client computation and communication time) and statistical heterogeneity (e.g., non-iid, quantity and quality of client data) on feature drift, providing theoretical guidance for understanding feature drift in federated learning.

\item To uncover causal relationships between the federated training process and feature drift, we designed a causal graph as a benchmark for causal inference of feature drift in federated learning. This graph models the causal relationships between sample features, classification results, and optimizers.

\item We introduce CAFE, a drift-aware federated learning approach based on causal inference. CAFE adjusts the classifier using the causal relationship between sample features and classification results, effectively extracting the drift direction through model parameters and momentum information. The algorithm asynchronously aggregates gradients uploaded by clients on the server side and performs causal training locally, thus learning causally invariant features. Notably, it does not require prior knowledge of the complex data distribution or communication frequency of clients during training. During inference, causal reasoning corrects drifts introduced by the global optimizer. Experimental evidence demonstrates that the proposed method outperforms previous state-of-the-art approaches under both general conditions and extreme settings of data and device heterogeneity.

\end{itemize}

The remainder of this paper is organized as follows. Section \ref{sec.2} reviews the related work. Section \ref{sec.3} introduces the terms and definitions of feature drift. The analysis of the mechanism of feature drift and the proposal of the CAFE method are presented in Sections \ref{sec.4} and \ref{sec.5}, respectively. Simulation results are provided in Section \ref{sec.6}, and conclusions and future work are presented in Section \ref{sec.7}. 

\section{Related Work}
\label{sec.2}
In this section, we concentrate on client-based methods that are closely related to our work while also providing a brief overview of alternative approaches. Comprehensive studies can be found in \cite{ref54}, \cite{ref55}, and \cite{ref57}.

\subsection{Feature Drift in Federated Learning} 
To address the challenge of feature drift in FL, there has been significant interest in developing methods to accurately identify and correct both the magnitude and direction of drift. These methods include sampling techniques, algorithm-centric approaches, and system-centric strategies. However, it is important to note that sampling does not necessarily guarantee an improvement in classification accuracy; in some cases, it may even lead to a degradation of model performance \cite{ref40}. 
Algorithm-centric methods (e.g. FedProx \cite{ref13} and SCAFFOLD\cite{ref23}) modify loss functions to enhance the sensitivity of classification algorithms towards minority classes. Nevertheless, these methods primarily focus on leveraging differences among data models to enhance performance without directly addressing the underlying causes of drift.
System-centric methods address class imbalance by modifying the structural framework of federated learning itself\cite{ref15,ref41}. However, they often necessitate prior knowledge regarding class distribution among clients or require sharing backbone networks—conditions which may prove impractical in real-world scenarios.

In addition, while existing research on participation imbalance primarily focuses on convergence or fairness issues, this area is also closely intertwined with feature drift. Typical solutions encompass client selection probability control methods and client capability adaptation methods. Client selection probability control methods \cite{ref34,ref35,ref36} ensure equitable participation opportunities for all clients \cite{ref42,ref43}.
Client capability adaptation methods dynamically adjust the federated learning model framework in accordance with client capabilities \cite{ref37,ref38,ref39,ref24}. However, they necessitate the deployment of additional scheduling algorithms on the server side, thereby increasing the server's computational load across diverse implementation environments.

To the best of our knowledge, there is limited research that comprehensively addresses class imbalance and participation imbalance in federated learning to mitigate global feature drift. ClassTer tackles this issue from the perspective of personalized federated learning (PFL)\cite{ref58}. It performs PFL on mobile devices through clustering and distillation methods, aiming to solve test-time data shifts in the presence of mobile device heterogeneity and data heterogeneity. However, while this solution has been successful in PFL, the clustering method is insufficient to reduce drift at its source in a centralized federated learning architecture and does not analyze how device and data heterogeneity lead to shifts.

\subsection{Causal Inference in Federated Learning}
Causal inference is a critical research direction in statistics and machine learning, aiming to identify and quantify causal relationships between variables from data. Causal intervention, an important subfield within causal inference, involves using "intervention" operations to block spurious correlations between variables, thereby revealing true causal effects. For instance, in human multimodal language understanding, the SuCI (Subject Causal Intervention) module \cite{ref51} employs backdoor adjustment theory to eliminate confounding effects of subjects on model predictions, significantly enhancing the model's generalization ability. Similarly, PACIFI (Preference-aware Causal Intervention) \cite{ref50} addresses the issue of user preferences being confounded in sequential recommendation systems through front-door adjustment and counterfactual data augmentation.

Several studies in federated learning have introduced causal intervention techniques to uncover causal relationships in data\cite{ref56}. FedCausal \cite{ref48} unifies local and global optimization into a fully directed acyclic graph (DAG) learning process with a flexible optimization objective, adaptively handling both homogeneous and heterogeneous data. CausalFL \cite{ref49} recovers the conditional distribution of missing confounders given observed confounders from dispersed data sources to identify causal effects. FedCSL \cite{ref47} employs a federated local-to-global learning strategy to address scalability issues. CausalRFF \cite{ref46} learns data similarity and causal effects from multiple dispersed data sources in a federated setting. However, these studies have not specifically focused on the causal relationships introduced by the federated learning framework itself, which is a key factor contributing to feature drift.

\section{Preliminaries and Problem Formulation}
\label{sec.3}
\subsection{Terminology}
We denote the training set with $N$ samples as $\{(x_i,y_i)\}_{i=1}^N$, where each sample consists of a embedding vector $x_i$ and its corresponding label $y_i$.Without loss of generality, the parameterized classification model $\boldsymbol{w}=\{\boldsymbol{\theta},\boldsymbol{\phi}\}$ can be decomposed into a feature extractor (all layers except the last, represented as $f_\theta{:}\mathcal{X}\to\mathcal{H}$ and a linear classifier (the last layer, represented as $f_\phi{:}\mathcal{H}\to\mathbb{R}^{[C]}$. Specifically, the feature extractor maps a sample $\boldsymbol{x}$ to a embedding vector $\boldsymbol{h}=f_{\boldsymbol{\theta}}(\boldsymbol{x})$ in the feature space $\mathcal{H}$, and the classifier generates a probability distribution $f_{\phi}\left(\boldsymbol{h}\right)$ as a prediction for $\boldsymbol{x}$.

\subsection{Federated Learning}
We consider a federated learning framework with \(K\) clients, where \(C\) being the total number of classes. Each client has its own dataset following the distribution \(\mathcal{P}_k(x, y)\), i.e., \(\mathcal{D}_k \sim \mathcal{P}_k(x, y)\) with \(n_k\) data samples. The global dataset is the union of all client datasets, denoted as \(\mathcal{D} = \cup_{k=1}^K \mathcal{D}_k \sim \mathbb{P}\) on \(\mathbb{R}^d \times \mathbb{R}\), with the total number of data samples being \(n = \sum_{k=1}^K n_k\).

Federated learning algorithms typically require multiple rounds of local and global processes to optimize the global model. At the beginning of each round, a subset of clients \(\mathcal{S} \subseteq \mathcal{K}\) is selected, where \(\mathcal{K}\) is the set of all clients. Then, each selected client \(k \in \mathcal{S}\) performs the following local process in parallel:

1. During the local training process, clients download the global parameters, i.e., \(\boldsymbol{w}_k \leftarrow \boldsymbol{w}_G\). We use the subscripts “\(k\)” and “\(G\)” to distinguish between local and global parameters.

2. Subsequently, clients update the downloaded model on their local training set \(\{(x_{k,i}, y_{k,i})\}_{i=1}^{N_k}\). Let \(\ell: \mathbb{R}^C \times \mathcal{Y} \to \mathbb{R}_+\) be the loss function, then \(\mathcal{L}_k(\boldsymbol{w}_G)\) is the expected local objective function on the \(k\)-th client's local dataset \(\mathcal{D}_k\):
\begin{equation}
\mathcal{L}_k(\boldsymbol{w}_G) = \mathbb{E}_{(x,y) \in \mathcal{D}_k} [\ell(\boldsymbol{w}_G; (x, y))].
\end{equation}

During the server aggregation process, the server collects the updated parameters from these clients and optimizes the combination of these local losses. We denote \(\{\hat{\boldsymbol{w}}_k\}\) as the parameter updates from the \(k\)-th client, and the server updates the global model via:
\begin{equation}
\boldsymbol{w}_G \leftarrow \frac{1}{|\mathcal{S}|} \sum_{k \in \mathcal{S}} \hat{\boldsymbol{w}}_k.
\end{equation}

In the standard formulation of FL, the goal is to minimize the global average of the local objectives:
\begin{equation}
\min_{\boldsymbol{w}_G \in \mathbb{R}^d} \mathcal{L}(\boldsymbol{w}_G) = \min_{\boldsymbol{w}_G \in \mathbb{R}^d} \sum_{k=1}^K p_k \mathcal{L}_k(\boldsymbol{w}_G),
\end{equation}
where \(p_k\) is the weight assigned to each client, satisfying \(p_k \geq 0\) and \(\sum_{k=1}^K p_k = 1\). \(\mathcal{L}_k(\boldsymbol{w}_G; \mathcal{D}_k)\) is the local loss as shown in Eq. (1), computed based on the local data \(\mathcal{D}_k\).

\subsection{feature drift in Federated Learning}
In this paper, feature drift refers to the deviation of the features extracted by the model (i.e., the embedding vector) due to changes in data distribution and the aggregation process in federated learning. It can be defined as the discrepancy between the embedding vector $\boldsymbol{h}$ and the optimal feature representation $\boldsymbol{h}^{*}$, that is, $\boldsymbol{h}\neq \boldsymbol{h}^{*}$.

In this section, we propose a taxonomy that systematically categorizes and describes the factors contributing to feature drift in federated learning from three distinct perspectives, whereas prior studies have only partially explored these aspects. These include:

$\cdot$ \textbf{Non-iid Data Distribution}: The data distribution for each client is denoted as  $P_{(.)}(x,y)$, where $x$ represents the features and $y$ represents the labels. In the non-iid scenario, the data distributions across different clients are not identical, i.e., $P_i(x,y)\neq P_j(x,y)$ for $i\neq j$.

Moreover, the non-iid nature of client data can manifest in several specific forms:

- Feature Distribution Skew: The feature distributions $P_{(.)}(x) $ differ across clients.

- Label Distribution Skew: The label distributions $P_{(.)}(y)$ vary among clients.

- Quantity Skew: The amount of data held by different clients is uneven, potentially causing data from some clients to disproportionately influence the global model.

$\cdot$ \textbf{Class Imbalance}: For class-imbalanced datasets, the label set \(K\) contains majority classes \(j_c \in \mathcal{J}_c\) and minority classes \(r_c \in \mathcal{R}_c\). We have \(\mathcal{J}_c \cap \mathcal{R}_c = \emptyset\), where \(\emptyset\) denotes the empty set. The number of training samples in each class satisfies \(n_j \gg n_r > 0\).

$\cdot$ \textbf{Participant Imbalance}: For a node set \(N\) with imbalanced communication frequencies, it includes Frequent Communicating Nodes \(j_n \in \mathcal{J}_n\) and Sparsely Communicating Nodes \(r_n \in \mathcal{R}_n\). Let \(\mathcal{N}_j\) and \(\mathcal{N}_r\) be the communication frequencies of frequent and sparse communicating nodes, respectively, satisfying \(\mathcal{N}_j \gg \mathcal{N}_r > 0\).

We designate the three factors above as federated learning drift factors (FLDF). When these factors coexist in the federated framework, according to \cite{ref37,ref3}, the performance of federated learning can sharply decline. Specifically, non-iid data leads to divergent update directions from clients, which are suboptimal for the test set. Class imbalance causes client updates to favor local majority classes. Participant imbalance results in global aggregation updates favoring frequently communicating clients. Collectively, these issues contribute to feature drift.

\section{Theoretical Analysis: The Mechanism Underlying feature drift}
\label{sec.4}
This section provides a comprehensive analysis of the drifts that emerge during the federated learning training phase. It introduces the drift parameters $\hat{\boldsymbol{d}^{R}_G}$ and $\hat{\boldsymbol{d}^{E}_k}$, which serve as a theoretical foundation for constructing the causal graph discussed in Section \ref{sec.5}. The analysis delves into how feature drift is triggered by both global and local optimizers, examining the underlying mechanisms, mathematical expressions, and the specific influence magnitude of each component.

\subsection{The Composition of feature drift}
To evaluate the model's generalization ability and prevent the shortcomings of the model in some categories from being masked, we refer to the definition of feature drift presented in Section \ref{sec.3}. Specifically, we observe the presence of feature drift by comparing the updates of the embedding vector in a system with various FLDF and in an ideal federated system.

In particular, consider an asynchronous federated system A that incorporates all types of FLDF, where client data distributions exhibit class-imbalance properties and are non-iid. During global aggregation, the server does not wait for straggling clients but aggregates directly after a fixed waiting time. Conversely, the ideal federated system B is a synchronous federated learning system where each client has a uniformly distributed dataset with the same total number of samples, and all clients participate in aggregation in each communication round, i.e., $P(y_i=c)=\frac{1}{C},P(k\in \mathcal{S})=\frac{1}{K}$.

For both federated systems, the classifier parameters of the $k$-th client, $\phi_{k}$, are represented as a $C$-dimensional weight vector $\{\phi_{k,c}\}_{c=1}^{C}$. Samples from the c-th class and other classes are respectively termed positive samples and negative samples. 

The embedding vector extracted by the feature extractor undergoes an affine transformation to generate the raw prediction scores (commonly referred to as logits) for each class. Here, we explicitly set the drift term of the affine transformation to zero, as it has negligible impact on classification performance. The softmax operator normalizes the logits (\( z_{i,c} = \boldsymbol{\phi}_c^T \boldsymbol{h}_i \)) and returns probability $p_{i,c}$ which means $i$-th sample belongs to class $c$:
\begin{equation}
    \begin{aligned}
        p_{i,c} = \text{softmax}(z_{i,c}) &= \frac{e^{z_{i,c}}}{\sum_{j=1}^{C} e^{z_{i,j}}}\\
        &=\frac{\exp(\boldsymbol{\phi}_{c}^\top\boldsymbol{h}_{i})}{\sum_{j=1}^{C}\exp(\boldsymbol{\phi}_{j}^\top\boldsymbol{h}_{i})},
    \end{aligned}
\end{equation}

We set the softmax input drift to zero, as it has minimal impact on classification performance. Using cross-entropy loss as the supervised loss, the supervised loss for the $i$-th client's classifier is expressed as:
\begin{equation}
    \begin{aligned}
        \mathcal{L}_{\sup_i}(\boldsymbol{\phi}_c,\boldsymbol{h}_i)&=\mathbb{E}_{(\boldsymbol{h}_i,y_i)\in\mathcal{D}_k}[l_{\sup_i}(\boldsymbol{\phi}_c;(\boldsymbol{h}_i,y_i))]\\
        &=\sum_{i=1}^N\sum_{c=1}^C\mathcal{I}\left\{y_i=c\right\}\log p_{i,c}\\
        &=\sum_{i=1}^N\sum_{c=1}^C\mathcal{I}\left\{y_i=c\right\}\log\frac{\exp\left(\boldsymbol{\phi}_{k,c}^\top\boldsymbol{h}_{i}\right)}{\sum_{j=1}^C\exp\left(\boldsymbol{\phi}_{k,j}^\top\boldsymbol{h}_{i}\right)}
    \end{aligned}
\end{equation}

where $\mathcal{I}\{\cdot\}$ is the indicator function.

Using the loss function to calculate the update of feature \( \boldsymbol{h}_{i} \), we can observe the presence of feature drift. The optimizer \( \boldsymbol{f}_{k} \) of the \( k \)-th client employs the momentum gradient descent (MGD) method to update the model, with a momentum decay rate of \( \mu_k \). As the model parameters are updated, the partial derivatives of the features change and exhibit a shift. We describe this as follows.

\begin{theorem}
(General Form of Feature Update) The update of features is influenced by the data distribution of the current batch, the local optimizer, and the global model, all of which collectively amplify the drift in feature updates. The general form of feature updates in both Systems A and B is: $-\sum_{i=1}^N\left(\mathcal{I}\left\{y_i=c\right\}-p_{i,c}\right)[\boldsymbol{\phi}^{(r,0)}_{k,c}
        -\eta(\boldsymbol{\nu}^{(1)}_{k,c}+\boldsymbol{\nu}^{(2)}_{k,c}+\cdots+\boldsymbol{\nu}^{(e)}_{k,c})]$.
\end{theorem}

\begin{proof}
Consider an input sample \(x_i\) entering the model training and its impact on feature updates. Assume, without loss of generality, that this sample belongs to class \(c\), i.e., \(y_i = c\). We compute the partial derivative of the loss function \(\mathcal{L}\) with respect to the feature vector \(\boldsymbol{h}_{i}\). Applying the chain rule, the feature vector update can be expressed as:

   \begin{equation}
    \begin{aligned}
        (\frac{\partial\mathcal{L}}{\partial\boldsymbol{h}_{i}})^{(e)}&=\frac{\partial\mathcal{L}}{\partial p_{i}}\cdot \frac{\partial p_{i}}{\partial z_{i}}\cdot \frac{\partial z_{i}}{\partial h_{i}}
    \end{aligned}
\end{equation}

Breaking down the derivative, we first compute the derivative of the cross-entropy loss \(\mathcal{L}\) with respect to the probability \(p_{i}\):
\begin{equation}
\begin{aligned}
    \frac{\partial\mathcal{L}}{\partial h_{i}}=
    &\sum_{i=1}^N\left(-\frac{\partial}{\partial p_{i,c}}\left(\mathcal{I}\{y_i=c\}\log p_{i,c}\right)\cdot \frac{\partial p_{i,c}}{\partial z_{i,c}}\cdot \frac{\partial z_{i,c}}{\partial h_{i}}\right.\\
    &\left.-\frac{\partial}{\partial p_{i,j}}\left(\sum_{j=1,j\neq c}^C\mathcal{I}\{y_i=j\}\log p_{i,j}\right)\cdot \frac{\partial p_{i,j}}{\partial z_{i,c}}\cdot \frac{\partial z_{i,c}}{\partial h_{i}}\right)\\
    =&\sum_{i=1}^N\left(-I\{y_i=c\}\frac{\partial\log p_{i,c}}{\partial p_{i,c}}\cdot \frac{\partial p_{i,c}}{\partial z_{i,c}}\cdot \frac{\partial z_{i,c}}{\partial h_{i}}\right.\\
    &\left.-\sum_{j=1,j\neq c}^CI\{y_i=j\}\frac{\partial\log p_{i,j}}{\partial p_{i,j}}\cdot \frac{\partial p_{i,j}}{\partial z_{i,c}}\cdot \frac{\partial z_{i,c}}{\partial h_{i}}\right)\\
    =&\sum_{i=1}^N\left(-\frac{\mathcal{I}\{y_i=c\}}{p_{i,c}}\cdot \frac{\partial p_{i,c}}{\partial z_{i,c}}\cdot \frac{\partial z_{i,c}}{\partial h_{i}}\right.\\
    &\left.-\sum_{j=1,j\neq c}^C\frac{\mathcal{I}\{y_i=j\}}{p_{i,j}}\cdot \frac{\partial p_{i,j}}{\partial z_{i,c}}\cdot \frac{\partial z_{i,c}}{\partial h_{i}}\right)
\end{aligned}
\end{equation}

Next, we compute the derivatives of the softmax probabilities \(p_{i,c}\) and \(p_{i,j}\) with respect to the logits \(z_{i,c}\):

\begin{equation}
\begin{aligned}
    \frac{\partial p_{i,c}}{\partial z_{i,c}}=&\frac{\partial}{\partial z_{i,c}}\left(\frac{e^{z_{i,c}}}{\sum_{j=1}^Ce^{z_{i,j}}}\right)\\
    =&\frac{e^{z_{i,c}}\sum_{j=1}^{C}e^{z_{i,j}}-e^{z_{i,c}}e^{z_{i,c}}}{(\sum_{j=1}^{C}e^{z_{i,j}})^{2}}\\
    =&\frac{e^{z_{i,c}}}{\sum_{j=1}^Ce^{z_{i,j}}}(1-\frac{e^{z_{i,c}}}{\sum_{j=1}^{C}e^{z_{i,j}}})\\
    =&p_{i,c}(1-p_{i,c})
\end{aligned}
\end{equation}

For \(\frac{\partial p_{i,j}}{\partial z_{i,c}}\), we have:

\begin{equation}
\begin{aligned}
    \frac{\partial p_{i,j}}{\partial z_{i,c}}=&\frac{\partial}{\partial z_{i,c}}\left(\frac{e^{z_{i,j}}}{\sum_{j=1}^Ce^{z_{i,j}}}\right)\\
    =&\frac{0\cdot\sum_{j=1}^{C}e^{z_{i,j}}-e^{z_{i,c}}e^{z_{i,j}}}{(\sum_{j=1}^{C}e^{z_{i,j}})^{2}}\\
    =&-p_{i,c}\cdot p_{i,j}
\end{aligned}
\end{equation}

Combining equations (8), (9), and (10), we obtain:

\begin{equation}
\begin{aligned}
    \frac{\partial\mathcal{L}}{\partial h_{i}}
    =&\sum_{i=1}^N\left(-\frac{\mathcal{I}\{y_i=c\}}{p_{i,c}}\cdot p_{i,c}(1-p_{i,c})\cdot \frac{\partial z_{i,c}}{\partial h_{i}}\right.\\
    &\left.-\sum_{j=1,j\neq c}^C\frac{\mathcal{I}\{y_i=j\}}{p_{i,j}}\cdot (-p_{i,c}\cdot p_{i,j})\cdot \frac{\partial z_{i,c}}{\partial h_{i}}\right)\\
    =&\sum_{i=1}^N\left(-\mathcal{I}\{y_i=c\}\cdot (1-p_{i,c})\cdot \frac{\partial z_{i,c}}{\partial h_{i}}\right.\\
    &\left.-\sum_{j=1,j\neq c}^C\mathcal{I}\{y_i=j\}\cdot (-p_{i,c})\cdot \frac{\partial z_{i,c}}{\partial h_{i}}\right)\\
    =&\sum_{i=1}^N\left(-\mathcal{I}\{y_i=c\}+p_{i,c}\sum_{j=1}^C\mathcal{I}\{y_i=j\}\right)\cdot \frac{\partial z_{i,c}}{\partial h_{i}}\\
    =&\sum_{i=1}^N(-\mathcal{I}\{y_i=c\}+p_{i,c})\cdot \frac{\partial z_{i,c}}{\partial h_{i}}
\end{aligned}
\end{equation}

Given \(z_{i,c} = \boldsymbol{\phi}_c^T \boldsymbol{h}_i\), it follows that:

\begin{equation}
\begin{aligned}
\frac{\partial z_{i,c}}{\partial h_{i}} = \boldsymbol{\phi}_{c}
\end{aligned}
\end{equation}

Combining equations (11) and (12), we derive the value of $(\frac{\partial\mathcal{L}}{\partial h_{i}})^{(r,e)}_k$ for the $k$-th client in the $r$-th communication round and the $e$-th training batch in federated learning:

\begin{equation}
\begin{aligned}
(\frac{\partial\mathcal{L}_k}{\partial h_{i}})^{(r,e)}= -\sum_{i=1}^N \left( \mathcal{I}\{y_i = c\} - p_{i,c} \right) \boldsymbol{\phi}^{(r,e)}_{k,c}
\end{aligned}
\end{equation}

The update formula for the local momentum gradient descent (MGD) optimizer is:
\begin{equation}
    \begin{aligned}
        m^{(r,e+1)}_k=\mu_G\cdot m^{(r,e)}_k+g_k^{(r,e+1)}\quad     \boldsymbol{w}^{(r,e)}_{k,c}=\boldsymbol{w}^{(r,e-1)}_{k,c}-\eta \boldsymbol{m}^{(r,e)}_{k,c}
    \end{aligned}
\end{equation}

Expanding $\boldsymbol{\phi}^{(e)}_{k,c}$ according to equation (12), we obtain:
\begin{equation}
    \begin{aligned}
        (\frac{\partial\mathcal{L}_k}{\partial\boldsymbol{h}_{i}})^{(r,e)}=&-\sum_{i=1}^N\left(\mathcal{I}\left\{y_i=c\right\}-p_{i,c}\right)[\boldsymbol{\phi}^{(r,e-1)}_{k,c} - \eta \boldsymbol{\nu}^{(r,e)}_{k,c}]\\
        =&-\sum_{i=1}^N\left(\mathcal{I}\left\{y_i=c\right\}-p_{i,c}\right)[\boldsymbol{\phi}^{(r,e-2)}_{k,c}\\
        &- \eta (\boldsymbol{\nu}^{(r,e-1)}_{k,c}+\boldsymbol{\nu}^{(r,e)}_{k,c})]\\
        =&\quad\cdots\cdots\\
        =&-\sum_{i=1}^N\left(\mathcal{I}\left\{y_i=c\right\}-p_{i,c}\right)[\boldsymbol{\phi}^{(r,0)}_{k,c}\\&
        -\eta(\boldsymbol{\nu}^{(r,1)}_{k,c}+\boldsymbol{\nu}^{(r,2)}_{k,c}+\cdots+\boldsymbol{\nu}^{(r,e)}_{k,c})]
    \end{aligned}
\end{equation}

This completes the proof.
\end{proof}

Among these terms, only $-\sum_{i=1}^N\left(\mathcal{I}\left\{y_i=c\right\}-p_{i,c}\right)$ is dependent on $i$, which we denote as \(\left(\frac{\partial\mathcal{L}_k}{\partial\boldsymbol{h}_{i}}\right)^{(r,e)}_{qtt} = -\sum_{i=1}^N \left(\mathcal{I}\{y_i = c\} - p_{i,c}\right)\), determined by the quantities of positive and negative samples in the current batch. The term $\boldsymbol{\phi}^{(r,0)}_{k,c}$ represents the initial model for the r-th communication round, which is the global model distributed by the server at the beginning of that round, is influenced by the global optimizer. Observing the term $-\eta(\boldsymbol{\nu}^{(1)}_{k,c}+\boldsymbol{\nu}^{(2)}_{k,c}+\cdots+\boldsymbol{\nu}^{(e)}_{k,c})$, it is evident that it is trained on each client's local data and thus captures the individual characteristics of each client. Consequently, Eq.(7) can be expressed as:
\begin{equation}
    \begin{aligned}
        (\frac{\partial\mathcal{L}_k}{\partial\boldsymbol{h}_{i}})^{(r,e)}= (\frac{\partial\mathcal{L}_k}{\partial\boldsymbol{h}_{i}})^{(r,e)}_{qtt}[(\frac{\partial\mathcal{L}_k}{\partial\boldsymbol{h}_{i}})^{(r,e)}_{glb}+ (\frac{\partial\mathcal{L}_k}{\partial\boldsymbol{h}_{i}})^{(r,e)}_{loc}]
    \end{aligned}
\end{equation}

To detect feature drift, we examine the differences in \((\frac{\partial\mathcal{L}_k}{\partial\boldsymbol{h}_{i}})^{(r,e)}_{qtt}\), \((\frac{\partial\mathcal{L}_k}{\partial\boldsymbol{h}_{i}})^{(r,e)}_{glb}\), and \((\frac{\partial\mathcal{L}_k}{\partial\boldsymbol{h}_{i}})^{(r,e)}_{loc}\) between systems A and B. 

First, for \((\frac{\partial\mathcal{L}_k}{\partial\boldsymbol{h}_{i}})^{(r,e)}_{qtt}\), if the \(i\)-th sample is a positive sample of class \(c\), \(\mathcal{I}\left\{y_i=c\right\}=1\) and \(\left(\mathcal{I}\left\{y_i=c\right\}-p_{i,c}\right)\) is positive. For negative samples of class \(c\), \(\mathcal{I}\left\{y_i=c\right\}=0\) and \(\left(\mathcal{I}\left\{y_i=c\right\}-p_{i,c}\right)<0\). Thus, more samples of class \(c\) lead to a larger \((\frac{\partial\mathcal{L}_k}{\partial\boldsymbol{h}_{i}})^{(r,e)}_{qtt}\).

Denote the terms related to the quantities of positive and negative samples in systems A and B that sampling from the current batch as \((\frac{\partial\mathcal{L}_k}{\partial\boldsymbol{h}_{i}})^{(r,e)}_{A_{qtt}}\) and \((\frac{\partial\mathcal{L}_k}{\partial\boldsymbol{h}_{i}})^{(r,e)}_{B_{qtt}}\). In system A, \(\mathbb{E}[(\frac{\partial\mathcal{L}_k}{\partial\boldsymbol{h}_{i}})^{(r,e)}_{A_{qtt}}]\) depends on the data distribution of the current client. In system B, with \(P(y_i=c)=\frac{1}{C}\), \(\mathbb{E}[(\frac{\partial\mathcal{L}_k}{\partial\boldsymbol{h}_{i}})^{(r,e)}_{B_{qtt}}]=\frac{N}{C}\). Clearly, \((\frac{\partial\mathcal{L}_k}{\partial\boldsymbol{h}_{i}})^{(r,e)}_{A_{qtt}} \neq (\frac{\partial\mathcal{L}_k}{\partial\boldsymbol{h}_{i}})^{(r,e)}_{B_{qtt}}\).

In the following two subsections, we will discuss the impact of global FLDF on the update of the embedding vector when only global factors are considered, and the influence of local FLDF on it when clients of the two systems are trained based on the global model of the same communication round.

\subsection{Feature drift from Global FLDF}

In this subsection, we focus on the impact of global FLDF on the update of the embedding vector \((\frac{\partial\mathcal{L}_k}{\partial\boldsymbol{h}_{i}})^{(r,e)}_{glb}=\boldsymbol{\phi}^{(0)}_{k,c}\) in systems A and B, and investigate the differences. Here, \(\boldsymbol{\phi}^{(0)}_{k,c}\) represents the model of the 0th batch in the rth communication round, which is sent by the server in the rth round, i.e., \(\boldsymbol{\phi}^{(0)}_i=\boldsymbol{\phi}^{(r)}_G\). In federated learning, global aggregation occurs at the end of each communication round, a process that can be analyzed to understand its impact on model parameters. Based on the global MGD optimizers, the global parameter update formula is given by:
\begin{equation}
\begin{aligned}
    m^{(r+1)}_G=\mu_G\cdot m^{(r)}_G+p^{(r)}g^{(r+1)}\quad w^{(r+1)}_G=w^{(r)}_G- m^{(r+1)}_G
\end{aligned}
    \end{equation}

Here, $g^{(r)}$ represents the gradients uploaded to the server by clients during the $r$-th communication round, and $p^{(r)}$ is the ratio of the data quantity from the client to the total global data volume.

Momentum, updated with a decay rate of $\mu_G$, adjusts $p^{(r)}$, which in turn updates the global model parameters. Gradients from various clients act as proxies, pulling the global parameters towards their respective client-specific directions. Clients that upload more frequently exert greater influence on the global model.

At the start of each new round, the server broadcasts the updated global model to all clients. As training progresses and the number of communication rounds increases, the global model becomes drift towards clients that frequently upload their local models. We describe the feature drift resulting from varying participation rates among clients as follows.

\begin{theorem}
(Feature Drift Caused by Client Heterogeneity) Influenced by the number of communications and the global optimizer, the general form of the influence from the global model during training is: $ (\frac{\partial\mathcal{L}_k}{\partial\boldsymbol{h}_{i}})^{(r,e)}_{glb}=\phi^{(r)}_G-\Sigma_{k=1}^{K}\Sigma_{s=1}^{\mathcal{S}}(\mu^{r-s+1}_G p^{(s-1)} \cdot \xi^{(s)}_{k})$.
\end{theorem}

\begin{proof}
Based on the global MGD optimizer update formula, the parameter update process can be expressed as:
\begin{equation}
    \begin{aligned}
        \phi^{(r+1)}_G=&\phi^{(r)}_G-\nu_G^{(r+1)}\\
        =&\phi^{(r)}_G-[\mu_G\nu_G^{(r)}+p^{(r)}\xi^{(r+1)}]\\
        =&\phi^{(r)}_G-[\mu_G(\mu_G\nu_G^{(r-1)}+p^{(r-1)}\xi^{(r)})+p^{(r)}\xi^{(r+1)}]\\
        =&\quad\cdots\cdots\\
        =&\phi^{(r)}_G-[p^{(r)}\cdot \xi^{(r+1)}\\
        &+\mu_G p^{(r-1)} \cdot \xi^{(r)}+\cdots+\mu^{r}_Gp^{(0)}\cdot \xi^{(1)}]
    \end{aligned}
\end{equation}

Expressing $\xi^{(r)}$, the gradients of the feature extractor uploaded by clients during the $r$-th communication round, in terms of client identifiers, we can rewrite it as:
\begin{equation}
    \begin{aligned}
        (\frac{\partial\mathcal{L}_k}{\partial\boldsymbol{h}_{i}})^{(r,e)}_{glb}&=\boldsymbol{\phi}^{(r,0)}_{k,c}\\
        &=\phi^{(r)}_G-\Sigma_{k=1}^{K}\Sigma_{s=1}^{\mathcal{S}}(\mu^{r-s+1}_G p^{(s-1)} \cdot \xi^{(s)}_{k})
            \end{aligned}
\end{equation}

This concludes the proof.
\end{proof}

Here, $\mu^{r-s+1}_G p^{(s-1)} \cdot \xi^{(s)}_{k}$ denotes the contribution of the $k$-th client during its $s$-th participation in aggregation.

Theorem 2 reveals that the influence of the $k$-th client on the direction of global model updates is related to their frequency of participation. The global model significantly aligns with the local models of frequently communicating clients, causing $\boldsymbol{h}_{i}^{(r)}$ to be closer to $\boldsymbol{h}^{(r)}_{J_c,i}$.

It is evident that the update form in system A is highly complex, involving the local optimal directions under non-iid conditions, participation order, frequency, and sample size of clients. In contrast, in system B, since all clients have the same sample size and participate in aggregation every round, \((\frac{\partial\mathcal{L}_k}{\partial\boldsymbol{h}_{i}})^{(r,e)}_{B_{glb}}\) can be expressed as:
\begin{equation}
    \begin{aligned}
        (\frac{\partial\mathcal{L}_k}{\partial\boldsymbol{h}_{i}})^{(r,e)}_{B_{glb}}
        &=\phi^{(r)}_G-\frac{1}{K}\Sigma_{k=1}^{K}\Sigma_{j=1}^{r}(\mu^{r-j+1}_G \cdot \xi^{(j)}_{k})
            \end{aligned}
\end{equation}

where \(\sum_{j=1}^{r}(\mu^{r-j+1}_G \cdot \xi^{(j)}_{k})\) represents the total contribution of the kth client participating in r rounds of aggregation. Under the influence of global FLDF, it is clear that \((\frac{\partial\mathcal{L}_k}{\partial\boldsymbol{h}_{i}})^{(r,e)}_{A_{glb}} \neq (\frac{\partial\mathcal{L}_k}{\partial\boldsymbol{h}_{i}})^{(r,e)}_{B_{glb}}\). If such a globally drift model is broadcasted without addressing imbalance, the drift would manifest in the first batch of each round and persistently affect subsequent training processes.

\subsection{feature drift from Local FLDF}

Considering the impact of local FLDF on the update of the embedding vector \((\frac{\partial\mathcal{L}_k}{\partial\boldsymbol{h}_{i}})^{(r,e)}_{loc}\) in systems A and B, the general form is \(-\eta(\boldsymbol{\nu}^{(1)}_{k,c}+\boldsymbol{\nu}^{(2)}_{k,c}+\cdots+\boldsymbol{\nu}^{(e)}_{k,c})\). We expand this in Theorem 3.

\begin{theorem}
(Feature Drift Caused by Data Distribution) The update of embedding vector is subject to the drift imposed by the optimizer. Specifically, when the training round is $e$, the general form of the influence from the $j$-th batch ($e\geq j$) is $-\eta\boldsymbol{\zeta}^{(j)}_{k,c}\frac{1-\mu^{e+1-j}}{1-\mu}$.
\end{theorem}

\begin{proof}
Further expanding $(\frac{\partial\mathcal{L}_k}{\partial\boldsymbol{h}_{i}})^{(r,e)}_{ioc}$ using the local MGD optimizer update formula yields:
   \begin{equation}
    \begin{aligned}
        &(\frac{\partial\mathcal{L}_k}{\partial\boldsymbol{h}_{i}})^{(r,e)}_{loc}\\=&-\eta(\boldsymbol{\nu}^{(r,1)}_{k,c}+\boldsymbol{\nu}^{(r,2)}_{k,c}+\cdots+\boldsymbol{\nu}^{(r,e)}_{k,c}\\
        =&-\eta[\boldsymbol{\zeta}^{(r,1)}_{k,c}+(\boldsymbol{\zeta}^{(r,2)}_{k,c}+\mu\boldsymbol{\zeta}^{(r,1)}_{k,c})\\
        &+\cdots+(\boldsymbol{\zeta}^{(r,e)}_{k,c}+\cdots+\mu^{e-1}\boldsymbol{\zeta}^{(r,1)}_{k,c})]\\
     =&-\eta[\boldsymbol{\zeta}^{(r,1)}_{k,c}(1+\mu+\cdots+\mu^{e-1})\\
     &+\boldsymbol{\zeta}^{(r,2)}_{k,c}(1+\mu+\cdots+\mu^{e-2})+\cdots+\boldsymbol{\zeta}^{(r,e)}]\\
     =&-\eta[\boldsymbol{\zeta}^{(r,1)}_{k,c}\frac{1-\mu^e}{1-\mu}+\boldsymbol{\zeta}^{(r,2)}_{k,c}\frac{1-\mu^{e-1}}{1-\mu}+\cdots+\boldsymbol{\zeta}^{(r,e)}_{k,c}\frac{1-\mu}{1-\mu}]
    \end{aligned}
\end{equation}

Observing the form of each term, the general form of the influence from the $e$-th batch ($e\geq j$) is  $-\eta\boldsymbol{\zeta}^{(j)}_{k,c}\frac{1-\mu^{e+1-j}}{1-\mu}$.

This completes the proof.
\end{proof}

We observe that during optimizer training, the gradient update information from historical batches in each communication round significantly influences the model's extracted features. Specifically, gradients from earlier batches exert a more substantial impact on the model, aligning with the principle of diminishing marginal utility.

To examine the deviation between $(\frac{\partial\mathcal{L}_k}{\partial\boldsymbol{h}_{i}})^{(r,e)}_{A_{loc}}$ and $(\frac{\partial\mathcal{L}_k}{\partial\boldsymbol{h}_{i}})^{(r,e)}_{B_{loc}}$, Lemma 1 establishes the relationship between the gradient updates from historical batches within the current communication round and the client data distribution.

\begin{lemma}
(Impact of Historical Batch Data Distribution on Feature Drift) The number of positive and negative samples from historical batches in the current communication round affects the gradient uploaded by each client during aggregation, thereby significantly influencing $(\frac{\partial\mathcal{L}_k}{\partial\boldsymbol{h}_{i}})^{(r,e)}_{loc}$. For all \(j \in [1, e]\), the influence of the \(j\)-th batch in the current communication round is
\begin{equation}
    \begin{aligned}
        \boldsymbol{\zeta}_{k,c}^{(r,j)}=\sum_{i=1,y_i=c}^{N_k}\left(1-p_{k,i,c}\right)\boldsymbol{h}_{k,i}^{(r,j)}-\sum_{i=1,y_i\neq c}^{N_k}p_{k,i,c}\boldsymbol{h}_{k,i}^{(r,j)}
    \end{aligned}
\end{equation}
\end{lemma}
\begin{proof}
We derive the form of gradients uploaded by each client in each communication round as shown in equation (12). For generality, the gradient uploaded by the $k$-th client in the $j$-th communication round is:
\begin{equation}
    \begin{aligned}
        \boldsymbol{\zeta}_{k,c}^{(r,j)}&=\frac{\partial\mathcal{L}}{\partial\boldsymbol{\phi}_{k,c}}\\
        &=\frac{\partial\mathcal{L}}{\partial p_c}\cdot \frac{\partial p_c}{\partial z_c}\cdot \frac{\partial z_c}{\partial \phi_{(c,i)}}\\
        &=-\sum_{i=1}^{N_k}\left(\mathcal{I}\left\{y_i=c\right\}-p_{i,c}\right)\boldsymbol{h}_{k,i}^{(r,j)}\\
        &=\sum_{i=1,y_i=c}^{N_k}\left(1-p_{k,i,c}\right)\boldsymbol{h}_{k,i}^{(r,j)}-\sum_{i=1,y_i\neq c}^{N_k}p_{k,i,c}\boldsymbol{h}_{k,i}^{(r,j)}
    \end{aligned}
\end{equation}

This concludes the proof.
\end{proof}

Observing the form of Eq. 12, we find that the magnitude of the training sample gradient \(g\) from the \(j\)-th batch is influenced by the proportion of positive and negative samples in that batch. For a sample \(i\) belonging to class \(c\), if \(c\) is a majority class (i.e., \(c \in J_c\)), the first term contains more items relative to a balanced dataset, which enhances the model's recognition accuracy for this class but may lead to overfitting. Meanwhile, due to the smaller number of negative samples for the majority class, the second term in Eq. 12 has fewer items compared to a balanced dataset, making the model's distinction between this class and others less clear. Conversely, although samples from the minority class \(R_c\) help the model better distinguish between classes, the limited number of positive samples prevents high recognition accuracy.

In System A, where clients have non-IID and imbalanced data distributions, each sampling is more likely to include the majority class within the client, while the minority class has fewer sampling opportunities. This means that during local training on each client, the expected gradient update direction depends on the majority and minority classes of that client. Consequently, the impact related to the local FLDF varies across clients in System A. In contrast, System B has uniformly distributed data across clients, leading to \(\left(\frac{\partial \mathcal{L}}{\partial \boldsymbol{h}_{k,c}}\right)^{(e)}_{A_{loc}} \neq \left(\frac{\partial \mathcal{L}}{\partial \boldsymbol{h}_{k,c}}\right)^{(e)}_{B_{loc}}\).

\subsection{Decomposition of embedding vectors}
In summary, each term in \((\frac{\partial\mathcal{L}_k}{\partial\boldsymbol{h}_{i}})^{(r,e)}_{A}\) and \((\frac{\partial\mathcal{L}_k}{\partial\boldsymbol{h}_{i}})^{(r,e)}_{B}\) differs, indicating the existence of feature drift. Moreover, during the process of updating the embedding vector using the MGD optimizer, the drift gradually increases.

Invariant features, denoted as \(h_{inv}\), refer to category features that remain stable across different client data distributions. These features are unaffected by the model training process and are orthogonal to feature components specific to the training process \cite{ref1}. 

Let $\boldsymbol{\lambda}^{(r)}_{G}=-\sum_{i=1}^N\left(\mathcal{I}\left\{y_i=c\right\}-p_{i,c}\right)\boldsymbol{\phi}^{(0)}_i$ and $\boldsymbol{\lambda}^{(r,e)}_{k}=\sum_{i=1}^N\left(\mathcal{I}\left\{y_i=c\right\}-p_{i,c}\right)\eta(\boldsymbol{\nu}^{(1)}_i+\boldsymbol{\nu}^{(2)}_i+\cdots+\boldsymbol{\nu}^{(r,e)}_i)$, where $r$ denotes the current communication round and $e$ represents the current training epoch. Define $\hat{\boldsymbol{d}^{(r)}_G}=\boldsymbol{\lambda}^{(r)}_{G}/\left\|\boldsymbol{\lambda}^{(r)}_{G}\right\|$ and $\hat{\boldsymbol{d}^{(r,e)}_k}=\boldsymbol{\lambda}^{(r,e)}_{k}/\left\|\boldsymbol{\lambda}^{(r,e)}_{k}\right\|$. 

According to our previous derivation, we have obtained the update form of the embedding vector. In system A with FLDF, it is evident that the embedding vector not only includes the invariant feature \(h_{inv}\) but also contains projections on the drift directions \(\hat{\boldsymbol{d}^{(r,e)}_k}\) and \(\hat{\boldsymbol{d}^{(r,e)}_k}\). Consequently, based on our earlier analysis, these drift projections can be approximated using parameters and momentum, thereby distinguishing them from the invariant components of the embedding vector, as stated in Proposition 1.

\begin{proposition}
(Invariant Feature Decomposition in FL)
Any embedding vector $\boldsymbol{h}$ can be decomposed into $\boldsymbol{h}=\boldsymbol{h}_{inv}+\boldsymbol{d}^{(r)}_G+\boldsymbol{d}^{(r,e)}_k$, where $\boldsymbol{h}_{inv}$ is orthogonal to the subspace spanned by $\boldsymbol{d}^{(r)}_G$ and $\boldsymbol{d}^{(r,e)}_k$. The components $\boldsymbol{h}_{inv}$, $\boldsymbol{d}^{(r)}_G$, and $\boldsymbol{d}^{(r,e)}_k$ correspond to the invariant part, the global optimizer-related part, and the local optimizer-related part, respectively, with $\boldsymbol{d}^{(r)}_G=\hat{\boldsymbol{d}^{(r)}_G}cos(\boldsymbol{h},\hat{\boldsymbol{d}^{(r)}_G})\|\boldsymbol{h}\|$ and $\boldsymbol{d}^{(r,e)}_k=\hat{\boldsymbol{d}^{(r,e)}_k}cos(\boldsymbol{h},\hat{\boldsymbol{d}^{(r,e)}_k})\|\boldsymbol{h}\|$.
\end{proposition}

Proposition 1 provides a theoretical foundation for characterizing the causal relationship between embedding vectors and classification results. In Proposition 1, we orthogonally decompose the embedding vector into invariant components, components related to the global optimizer, and components related to the local optimizer. The invariant component is the decisive factor affecting the classification result, while the components related to the global optimizer and the local optimizer are negative factors. The orthogonality between \(\boldsymbol{h}_{inv}\) and \(\boldsymbol{d}^{(r)}_G + \boldsymbol{d}^{(e)}_k\) reflects their independence.
\begin{equation}
    \begin{aligned}
        \boldsymbol{h}=\boldsymbol{h}_{inv}+\boldsymbol{d}^{(r)}_G+\boldsymbol{d}^{(e)}_k,\quad \boldsymbol{h}_{inv}\perp (\boldsymbol{d}^{(r)}_G+\boldsymbol{d}^{(e)}_k)
    \end{aligned}
\end{equation}

Observing the forms of \(\hat{\boldsymbol{d}^{(r)}_G}\) and \(\hat{\boldsymbol{d}^{(e)}_k}\), we find that many coefficient terms are canceled out. This implies that the samples from the current batch \(\left(\frac{\partial \mathcal{L}}{\partial \boldsymbol{h}_{k,c}}\right)^{(e)}_{qtt}\) have no direct influence on the direction of drift.

\section{CAFE: a Causal perspective on drift-Aware Federated lEarning}
\label{sec.5}
To systematically investigate and mitigate the impact of feature drift on classification predictions, this section examines the causal relationships between the invariant features $H_{inv}$ of samples, the sample features $H$, and the classification outcomes $Y$. Building upon the analysis in Section 4, we construct a causal graph for the training and inference phasees in federated learning and use this graph to correct classification predictions.
\begin{figure}
    \centering
    \includegraphics[width=1\linewidth]{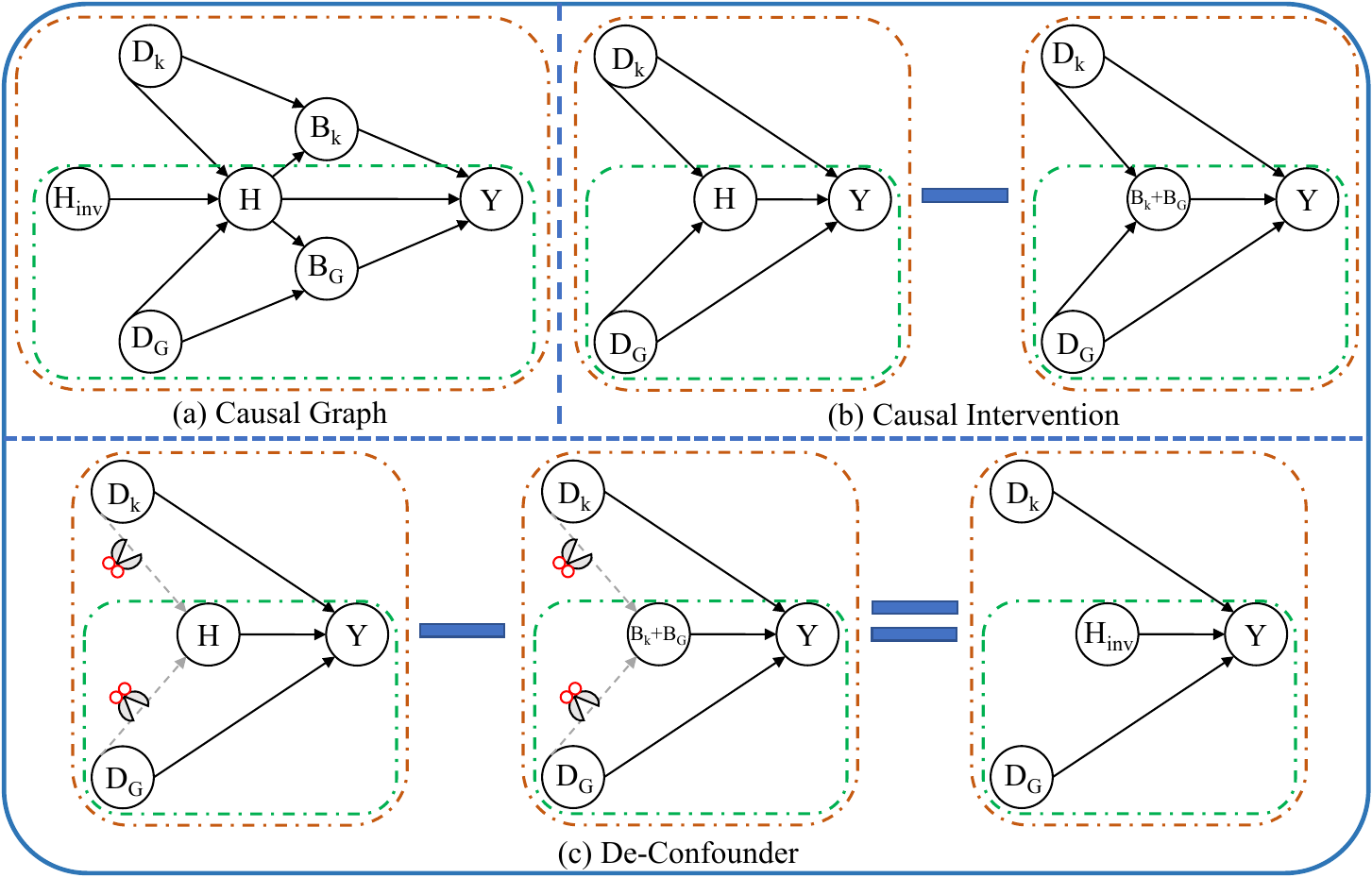}
    \caption{The causal graph of FL (a) and the process of implementing causal interventions in (b) and (c).}
    \label{fig:enter-label}
\end{figure}

\subsection{A causal perspective on FL}
Based on the assertion that feature drift is influenced by both global and local optimizers, we construct a causal graph with seven variables as shown in Figure 1(a): $F_k$ and $F_G$ represent the local and global optimizers, respectively; $H$ denotes the features extracted by the model from samples; $H_{inv}$ signifies the invariant components within the embedding vectors; $D_k$ and $D_G$ represent the global and local components embedded within the embedding vectors, which correspond to the drifts inherent in the feature representations; and $Y$ is the model's prediction. The causal graph, a directed acyclic graph (DAG), illustrates how these variables interact through causal relationships. Specifically, it reveals the following causal relationships, as shown in Figure 2(a):

\textbf{$(H_{inv},F_G,F_k)\rightarrow H$}. This indicates that the feature vector $H$ is influenced by the invariant component $H_{inv}$, as well as the global FLDF $F_G$ and the local FLDF $F_k$, as established by Theorem 1.

\textbf{$(F_G,H)\rightarrow D_G$} and \textbf{$(F_k,H)\rightarrow D_k$}. During model iteration, the global optimizer $F_G$ and the local optimizer $F_k$ lead to drift components $D_G$ and $D_k$ in the feature vector, with their magnitude determined by the magnitude of $H$, as shown in Proposition 1.

\textbf{$M\rightarrow D_k\rightarrow Y$} and \textbf{$M\rightarrow D_G\rightarrow Y$}. Both the global optimizer $F_G$ and the local optimizer $F_k$ result in global drift components $D_G$ and local drift components $D_k$, which subsequently affect the model's judgment of the label $Y$ for sample $i$:
\begin{equation}
    \begin{aligned}
        \boldsymbol{z}_{i,c}=\boldsymbol{\phi}_{c}^\top\boldsymbol{h}_{i}=\boldsymbol{\phi}_{c}^\top(\boldsymbol{h}_{inv_i}+\boldsymbol{d}^{(r)}_G+\boldsymbol{d}^{(r,e)}_k)
    \end{aligned}
\end{equation}

\textbf{$H\rightarrow (D_G,D_k)\rightarrow Y$} and \textbf{$H\rightarrow Y$}.  The feature vector $H$ contains both the invariant component $H_{inv}$ and the drift components $(D_G, D_k)$, leading to potential misclassification of $Y$.

$P(Y=i|H=\boldsymbol{h}, F_k=\boldsymbol{f}_k, F_G=\boldsymbol{f}_G)$ signifies that during federated training, the neural network calculates logits and normalizes them to obtain softmax probabilities based on the feature vector under the influence of global and local optimizers. We further explore this causal relationship from sample features to prediction results.

\subsection{Causal Intervention}
Observing the causal graph, we identify $F_k$ and $F_G$ as confounders of the $H\rightarrow Y$ relationship. A confounder is a variable that influences both the independent and dependent variables, creating a spurious statistical correlation between $H$ and $Y$. Meanwhile, $D_k$ and $D_G$ act as intermediaries, through which the independent variable exerts an indirect effect on the dependent variable, altering the total effect of $H\rightarrow Y$. Based on Figure 2(a), the goal of drift-aware classification is to isolate the direct causal effect along $H_{inv} \rightarrow Y$ and replace $p_{i,c}$ in equation (5).

By integrating the causal effects of confounding factors $F_G$ and $F_k$ through weighting, we eliminate their impact and obtain the total average causal effect of the feature vector $H$ on the classification result $Y$ using the do operator. However, we also need the direct causal effect of the invariant component $H_{inv}$ on $Y$. Given that $H$ contains both invariant and optimizer-related components, we subtract the indirect average causal effect caused by $D_k$ and $D_G$ from the total average causal effect, as illustrated in Figure 2(b). This yields the drift-aware classification result $\hat p_{i,c}^{train}$:
\begin{equation}
\begin{aligned}
    &\hat p_{i,c}^{train}\\ &= P(Y=i|do(H=\boldsymbol{h}), F_k=\boldsymbol{f}_k, F_G=\boldsymbol{f}_G) \\
&\quad- P(Y=i|do(H=\boldsymbol{d}_k+\boldsymbol{d}_G), F_k=\boldsymbol{f}_k, F_G=\boldsymbol{f}_G)
\end{aligned}
\end{equation}

Where $P\Big(Y_i=c|do(H=\boldsymbol{h}_i)\Big)$ signifies the distribution of $Y$ if everyone in the population were to fix $H$ at $\boldsymbol{h}_i$. It is important to note that, due to the nature of federated learning, models are trained locally at each client and use the global model directly during inference, no longer subject to local client influences. Thus, as depicted in the green dashed box in Figure 1, the mediation and confounding effects need not be considered during the inference phase:
\begin{equation}
    \begin{aligned}
    \hat p_{i,c}^{infer}&=P\Big(Y_i=c|do(H=\boldsymbol{h}_i), F_G=\boldsymbol{f}_G\Big)\\
    &\quad-P\Big(Y_i=c|do(H=\boldsymbol{d}^{(r)}_G), F_G=\boldsymbol{f}_G\Big)
    \end{aligned}
\end{equation}

Based on the probability $p_{i,c}$, the classification result can be obtained:
\begin{equation}
    \begin{aligned}
        y_i=\arg\max_{c\in C} p_{i,c}
    \end{aligned}
\end{equation}

\begin{figure*}
    \centering
    \includegraphics[width=1\linewidth]{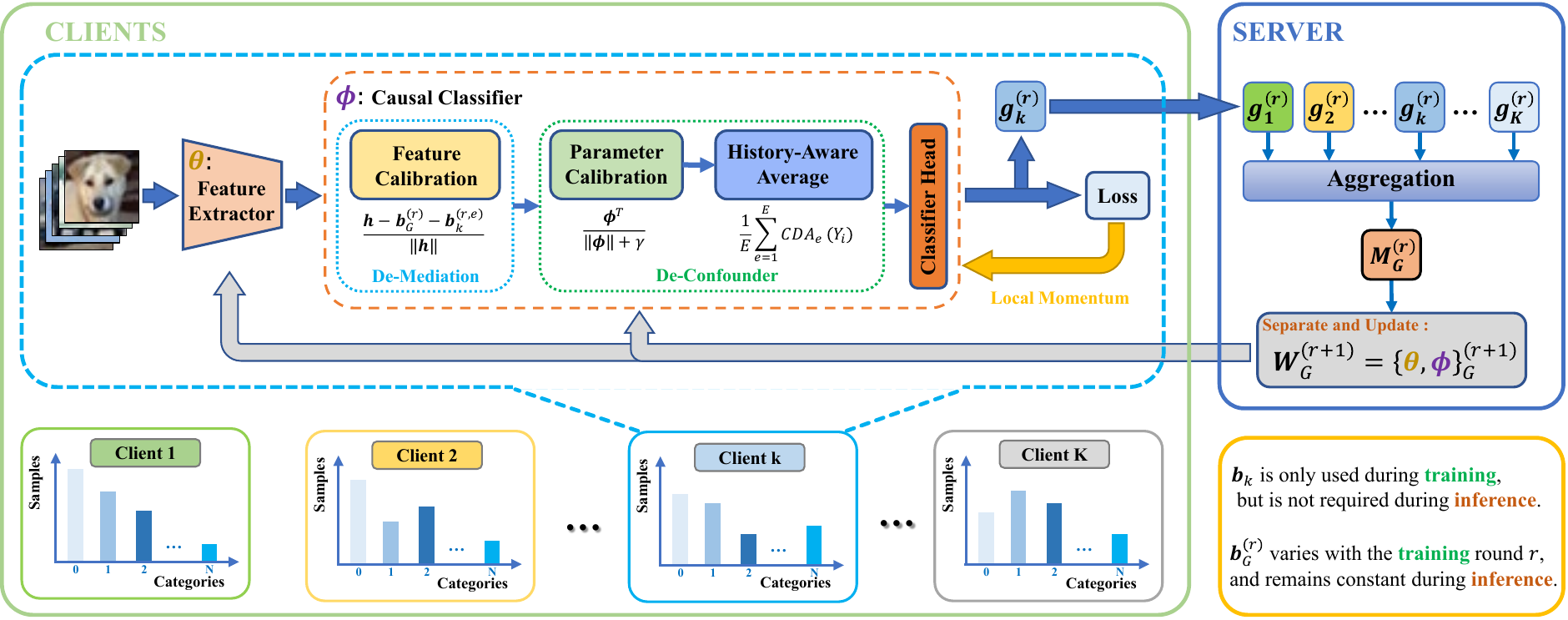}
    \caption{The detailed process of the CAFE algorithm during the training phase involves calibrating the embedding vectors and classifier parameters. Subsequently, historical information is integrated to achieve robust and drift-aware classification results for complex scenarios.}
    \label{fig:enter-label}
\end{figure*}

\subsection{De-Confounder Intervention}

We further analyze the two backdoor paths influenced by confounders $F_k$ and $F_G$ in the causal graph. From the perspective of causal inference, Mk and MG are essentially considered confounders that mislead the model to learn classification correlations specific to the federated learning training process in the training data, leading to prediction drifts for new samples in the inference stage.

According to the causal graph in Figure 2(a), the output probability $P(Y|H)$ is affected by the backdoor effect, which can be decomposed as follows using Bayes' rule:

\begin{equation}
    \begin{aligned}
&P\Big(Y_i=c|H=\boldsymbol{h}_i, F_k=\boldsymbol{f}_k, F_G=\boldsymbol{f}_G\Big)\\
&=\sum_{f_G,f_k}P(Y_i=c|H=\boldsymbol{h}_i,F_G=\boldsymbol{f}_G,F_k=\boldsymbol{f}_k)\\
&\qquad P(F_G=\boldsymbol{f}_G,F_k=\boldsymbol{f}_k|H=\boldsymbol{h}_i)
\end{aligned}
\end{equation}

Confounding factors Mk and MG introduce observational drifts via $P(F_G=\boldsymbol{f}_G,F_k=\boldsymbol{f}_k|H=\boldsymbol{h}_i)$. The $do(·)$ operator serves as an effective approximation for empirical intervention. Backdoor adjustment implies measuring the causal effects at each level of the confounders and then integrating them according to the output probabilities from different training stages to estimate the causal effect. As shown in Figure 2(b), the influences from $F_k$ and $F_G$ to H are cut off. The formula with intervention for Eq. (23) can be written into a do-free expression:
\begin{equation}
\label{eq18}
    \begin{aligned}
&P\Big(Y_i=c|do(H=\boldsymbol{h}_i), F_k=\boldsymbol{f}_k, F_G=\boldsymbol{f}_G\Big)\\
&=\sum_{f_G,f_k}P(Y_i=c|H=\boldsymbol{h}_i,F_G=\boldsymbol{f}_G,F_k=\boldsymbol{f}_k)\\
&\quad\cdot P(F_G=\boldsymbol{f}_G)P(F_k=\boldsymbol{f}_k)
\end{aligned}
\end{equation}

This equation is derived from the adjustment formula for $H$. It is evident that once we have the conditional probabilities mentioned, we can predict the effect of intervention according to this equation.

Equation of momentum indicate that, given fixed hyperparameters $\mu$ and $\eta$, each sample $F_G=\boldsymbol{f}_G$ and $F_k=\boldsymbol{f}_k$ is a function of model initialization and mini-batch sampling strategy, implying an infinite number of samples for $F_G$ and $F_k$. Fortunately, the inverse probability weighting formula in Eq.\ref{eq18} provides a method for approximation. For a finite dataset, we can only observe one $\boldsymbol{f}$ influencing the training within a batch, with the number of $\boldsymbol{f}_{(k)}$ equal to the number of trained batches. This is apparent in the federated learning training phase: after receiving the aggregated model from the server, the initial model of the first batch in each round is influenced by the global optimizer, and subsequent training is affected by the local optimizer.

Furthermore, for each client, the initialization model for each round is sent by the server, and each training phase starts anew with a different initialization model, unaffected by training phasees in other rounds. During training, $f_G$ remains constant, i.e., $P(F_G=\boldsymbol{f}_G)=1$, while $f_k$ is updated in each batch. Assuming $E$ batches have participated in training within this round, the equation can be simplified as follows:
\begin{equation}
    \begin{aligned}
&P\Big(Y_i=c|do(H=\boldsymbol{h}_i), F_k=\boldsymbol{f}_k, F_G=\boldsymbol{f}_G\Big)\\
&=\frac{1}{E}\sum_{E}P(Y_i=c|H=\boldsymbol{h}_i,F_G=\boldsymbol{f}^{(r)}_G,F_k=\boldsymbol{f}^{(r,e)}_k)\\
&=\frac{1}{E}\sum_{E}P_e(Y_i)
\end{aligned}
\end{equation}

Consider the probability $P(Y_i=c|H=\boldsymbol{h}_i,F_G=\boldsymbol{f}_G,F_k=\boldsymbol{f}_k)$, which represents the likelihood that the classifier identifies the feature extractor's output $\boldsymbol{h}_i$ as class $i$ under the influence of specific $f_G$ and $f_k$. 

The cosine classifier measures the relationship between feature vectors and class weights using cosine similarity, which can make intra-class samples more compact and inter-class samples more separated. Moreover, it avoids the exponential operations in softmax, thus providing more numerical stability. Additionally, the classification effect can be further optimized by tuning hyperparameters. In Eq. (14), we model $p_{i,c}^{train}$ as the output probability of the cosine classifier, where $\boldsymbol{h}^{(e)}$ and $\boldsymbol{w}_i^{(e)}$ represent the feature and parameter values influenced by $f_G$ in the $e$-th batch.
\begin{equation}
    \begin{aligned}
    p_{i,c}^{train}
    &\propto \tau\frac{f(\boldsymbol{h}_i;\boldsymbol{w}^{(r,e)})}{g(\boldsymbol{h}_i;\boldsymbol{w}^{(r,e)})}
    \end{aligned}
\end{equation}

Here, $\tau$ is a positive proportionality constant. Recall Assumption 1: $\boldsymbol{h}_i=\boldsymbol{h}_{inv_i}+\boldsymbol{d}^{(r)}_G+\boldsymbol{d}^{(r,e)}_k$. The numerator represents the original logits of the neural network: 
$f(\boldsymbol{h}_i;\boldsymbol{w}^{(r,e)}) = (\boldsymbol{w}^{(r,e)})^\top\boldsymbol{h}_i=(\boldsymbol{w}^{(r,e)})^\top(\boldsymbol{h}_{inv_i}+\boldsymbol{d}^{(r)}_G+\boldsymbol{d}^{(r,e)}_k)$.
The denominator is the propensity score, used to normalize each effect. Based on previous analysis, the model's parameters may also shift due to feature drift. To normalize this, $g(\boldsymbol{h}_i;\boldsymbol{w}^{(r,e)})$ is set to 
$g(\boldsymbol{h}_i;\boldsymbol{w}^{(r,e)})=\|\boldsymbol{h}_i\|\cdot\|\boldsymbol{w}^{(r,e)}\|+\gamma\|\boldsymbol{h}_i\|$.
By introducing the penalty term $\gamma$, the magnitude of$\|\boldsymbol{w}^{(r,e)}\|$ is adjusted, normalizing each component of $\boldsymbol{h}_i$.

In summary, the logit calculation for $P(Y = i|do(H = \boldsymbol{h}_i))$ in the $k$-th client can be expressed as:
\begin{equation}
\begin{aligned}
        &P\Big(Y_i=c|do(H=\boldsymbol{h}_i), F_k=\boldsymbol{f}_k, F_G=\boldsymbol{f}_G\Big)\\&=\frac\tau {E}\sum_{{e}=1}^{E}\frac{(\boldsymbol{\phi}_c^{(r,e)})^\top\boldsymbol{h}_{inv_i}}{(\|\boldsymbol{\phi}_c^{(r,e)}\|+\gamma)\|\boldsymbol{h}_i\|}\\
    &=\frac\tau {E}\sum_{{e}=1}^{E}\frac{(\boldsymbol{\phi}_c^{(r,e)})^\top(\boldsymbol{h}_{i}-\boldsymbol{d}_{G}^{(r)}-\boldsymbol{d}_k^{(r,e)})}{(\|\boldsymbol{\phi}_c^{(r,e)}\|+\gamma)\|\boldsymbol{h}_i\|}
\end{aligned}
\end{equation}

Similarly, we have
\begin{equation}
\begin{aligned}
        &P\Big(Y=i|do(H=\boldsymbol{d}_k+\boldsymbol{d}_G), F_k=\boldsymbol{f}_k, F_G=\boldsymbol{f}_G\Big)&\\
        &=\frac\tau {E}\sum_{{e}=1}^{E}\frac{(\boldsymbol{\phi}_i^{(e)})^\top(\boldsymbol{d}_k^{(e)}+\boldsymbol{d}_G^{(e)})}{(\|\boldsymbol{\phi}_i^{(e)}\|+\gamma)\|\boldsymbol{h}^{(e)}\|}\\
    &=\frac\tau {E}\sum_{{e}=1}^{E}\frac{(\boldsymbol{\phi}_i^{(e)})^\top(\hat{\boldsymbol{d}^{(r)}_G}cos(\boldsymbol{h},\hat{\boldsymbol{d}^{(r)}_G})\|\boldsymbol{h}\|+\hat{\boldsymbol{d}^{(e)}_k}cos(\boldsymbol{h},\hat{\boldsymbol{d}^{(e)}_k})\|\boldsymbol{h}\|)}{(\|\boldsymbol{\phi}_i^{(e)}\|+\gamma)\|\boldsymbol{h}^{(e)}\|}\\
    &=\frac\tau {E}\sum_{{e}=1}^{E}\frac{(\boldsymbol{\phi}_i^{(e)})^\top(\hat{\boldsymbol{d}^{(r)}_G}cos(\boldsymbol{h},\hat{\boldsymbol{d}^{(r)}_G})+\hat{\boldsymbol{d}^{(e)}_k}cos(\boldsymbol{h},\hat{\boldsymbol{d}^{(e)}_k}))}{\|\boldsymbol{\phi}_i^{(e)}\|+\gamma}
\end{aligned}
\end{equation}

\subsection{Proposed CAFE Local Training Framework}
By substituting $\boldsymbol{h}_inv$ from Equation (20) into $\boldsymbol{h}$ in Equation (27), we eliminate the indirect effects within the formulation. The detailed algorithmic description of CAFE is presented in Alg. 1. In summary, during the training phase, to obtain an unbiased estimate of \( Y \), we replace \( p_{i,c} \) in eq.(5) with the $p_{i,c}^{train}$:
\begin{subequations}
\begin{align}
    \hat p_{i,c}^{train}
    &=\frac{\tau} {E}\sum_{{e}=1}^{E}\left(\frac{(\boldsymbol{w}^{(r,e)})^\top\boldsymbol{h}_i}{(\|\boldsymbol{w}^{(r,e)}\|+\gamma)\|\boldsymbol{h}_i\|}\right.\nonumber\\
    &\left.\quad-\alpha\cdot\frac{cos(\boldsymbol{h}_i,\boldsymbol{\hat{d}_{G}}^{(r)})\cdot(\boldsymbol{w}^{(r,e)})^\top\boldsymbol{\hat{d}}_{G}}{\|\boldsymbol{w}^{(r,e)}\|+\gamma}\right.\nonumber\\
    &\left.\quad-\beta\cdot\frac{cos(\boldsymbol{h}_i,\boldsymbol{\hat{d}}_{k})\cdot(\boldsymbol{w}^{(r,e)})^\top\boldsymbol{\hat{d}}_{k}}{\|\boldsymbol{w}^{(r,e)}\|+\gamma}\right)
    \end{align}
\end{subequations}

where \( \alpha \) and \( \beta \) are linear balancing parameters between direct and indirect effects. \( p_{i,c} \) represents the corrected probability that sample \( i \) belongs to class \( c \), which significantly differs from the expression in eq.(5). Eq.(22b) substitutes $\boldsymbol{d}$ into the second and third terms and cancels out $\boldsymbol{h}_i$ in both the numerator and the denominator.

\begin{algorithm}
\caption{CAFE (Proposed Framework in training phase)}\label{alg:alg1}
\begin{algorithmic}
\STATE 
\STATE {\textbf{Input: }}Learning rate $\eta$, local batch $E$, client number $K$, client sample number $K^\ast$,class number $C$, server momentum decay rate $\mu_G$, local momentum decay rate $\mu_k$
\STATE {\textbf{Initialize: }}Global model $\boldsymbol{w}^{(0)}_G=\{\theta,\phi\}$, server momentum buffer $\boldsymbol{f}^{(0)}_G=\boldsymbol{0}$. $\forall k \in [K], 
 \forall r\leq R$, local momentum buffer $\boldsymbol{f}^{(r,0)}_k=\boldsymbol{0}$
\STATE \textbf{for } each round $r=1,\cdots,R$ \textbf{do}
\STATE \hspace{0.5cm}\textbf{Client }$k \in [K]$ \textbf{in parallel do}
\STATE \hspace{1cm}Receive $\boldsymbol{w}^{(r)}$ to initialize $\boldsymbol{w}_k^{(r,0)}$.
\STATE \hspace{1cm}\textbf{for } each mini-batch $e=0,1,\cdots,E-1 $\textbf{ do}
\STATE \hspace{1.5cm}Calculate sample features $\boldsymbol{h}_i$
\STATE \hspace{1.5cm}$\%causal\quad drift-aware \quad classifier\%$
\STATE \hspace{1.5cm}Feature Calibration: $\frac{\boldsymbol{h}_i-\boldsymbol{d}_G^{(r)}-\boldsymbol{d}_k^{(r,e)}}{\|\boldsymbol{h}_i\|}$
\STATE \hspace{1.5cm}Parameter Calibration:$\frac{\phi^\top}{\|\phi\|+\gamma}$
\STATE \hspace{1.5cm}Calculate $P_e(Y_i)=\frac{\boldsymbol{h}_i-\boldsymbol{d}_G^{(r)}-\boldsymbol{d}_k^{(r,e)}}{\|h\|}\cdot \frac{\phi^\top}{\|\phi\|+\gamma}$
\STATE \hspace{1.5cm}History-Aware Average $p_{i,c}^{train}=\frac{1}{E}\sum_{e=1}^{E}P_e\left(Y_i\right)$
\STATE \hspace{1.5cm}Calculate the calibration loss and compute mini-batch gradient $\nabla F(\boldsymbol{w}_k^{(r,e)})$
\STATE \hspace{1.5cm}$\boldsymbol{f}_k^{(r,e+1)}=\mu_k\boldsymbol{f}_k^{(r,e)}+\nabla F(\boldsymbol{w}_k^{(r,e)})$
\STATE \hspace{1.5cm}$\boldsymbol{w}_k^{(r,e+1)}=\boldsymbol{w}_k^{(r,e)}-\eta\boldsymbol{f}_k^{(r,e+1)}$
\STATE \hspace{1cm}\textbf{end for}
\STATE \hspace{0.5cm} \textbf{Server:}
\STATE \hspace{1cm}Receive $\boldsymbol{g}_k^{(r)}$ from clients.
\STATE \hspace{1cm}$\boldsymbol{f}_G^{(r+1)}=\mu_G\boldsymbol{f}_G^{(r)}+\frac{1}{K^\ast}\sum_{k=1}^{K^\ast}\boldsymbol{g}_k^{(r)}$
\STATE \hspace{1cm}$\boldsymbol{w}_G^{(r+1)}=\boldsymbol{w}_G^{(r)}-\eta\boldsymbol{f}_G^{(r+1)}$
\STATE \hspace{1cm}Send $\boldsymbol{w}_G^{(r+1)}$ to clients.
\STATE \textbf{end for}
\STATE \textbf{Output: }Causal drift-aware global model $\boldsymbol{w}^{(R)}_G$
\end{algorithmic}
\label{alg1}
\end{algorithm}

In the proposed CAFE algorithm, the updated causal classifier is structured into three modules: the Feature Calibration module, which refines the features $\boldsymbol{h}_i$; the Parameter Calibration module, which adjusts the parameters $\boldsymbol{\phi}$; and the History-Aware Average module, which incorporates historical update directions. As shown in the training phase algorithm diagram in Fig. 3, detailed descriptions of each module are provided in Algorithm 1.

In each epoch, the client performs local training using the causal classifier described above, based on the model received from the server. Once local training is complete, the client uploads the updated model gradients to the parameter server for global aggregation. If the federated framework employs an asynchronous aggregation strategy, drift may still persist in the globally aggregated model. This drift needs to be mitigated during inference, a process that will be detailed in the following subsection.

\subsection{Inference Based on the Global Model}

The inference phase is depicted within the green dashed box in Fig. 2. Unlike training stage, federated learning directly utilizes the global model for inference. It is important to note that the global model is aggregated from "undrifted" models trained by each client using causal classifiers. At this stage, the global model is influenced only by feature drift caused by the number of client communications. Therefore, when using the global model for inference, the $p_{i,c}^{infer}$ only needs to eliminate the mediator and confounder effects accumulated during training due to the global optimizer.

\begin{algorithm}
\caption{CAFE (Proposed Framework in inference phase)}\label{alg:alg2}
\begin{algorithmic}
\STATE 
\STATE {\textbf{Input: }}Data $\boldsymbol{s}^{(test)}$, global model $\boldsymbol{w}^{(R)}=\{\theta,\phi\}^{(R)}$
\STATE \hspace{0.5cm}Calculate sample features $\boldsymbol{h^{(test)}}$
\STATE \hspace{0.5cm}Feature Calibration: $\frac{\boldsymbol{h}_i-\boldsymbol{d}_G^{(R)}}{\|\boldsymbol{h}_i\|}$
\STATE \hspace{0.5cm}Parameter Calibration:$\frac{\phi^\top}{\|\phi\|+\gamma}$
\STATE \hspace{0.5cm}Calculate $p_{i,c}^{infer}=\frac{\boldsymbol{h}_i-\boldsymbol{d}_G^{(R)}}{\|\boldsymbol{h}_i\|}\cdot \frac{\phi^\top}{\|\phi\|+\gamma}$
\STATE \hspace{0.5cm}Predicted 
classes $y_i=\arg\max_{i\in C} p_{i,c}^{infer}$
\STATE \textbf{Return }$y_i$
\end{algorithmic}
\label{alg2}
\end{algorithm}

Based on the causal diagram, the backdoor intervention yields:

\begin{equation}
\begin{aligned}
    &P\Big(Y_i=c|do(H=\boldsymbol{h}_i), F_k=\boldsymbol{f}_k, F_G=\boldsymbol{f}_G\Big)\\
    &=\sum_{f_G}P(Y_i=c|H=\boldsymbol{h}_i,F_G=\boldsymbol{f}_G)P(F_G=\boldsymbol{f}_G)
\end{aligned}
\end{equation}

Since $\boldsymbol{w}_G$ and $\boldsymbol{f}_G$ are constants during inference, equal to $\boldsymbol{w}_G^{(R)}$ and $\boldsymbol{f}_G^{(R)}$ at the end of training, respectively, we have $P(F_G=\boldsymbol{f}_G)=1$. Thus,
\begin{equation}
\begin{aligned}
    &P\Big(Y_i=c|do(H=\boldsymbol{h}_i), F_k=\boldsymbol{f}_k, F_G=\boldsymbol{f}_G\Big)\\
    &=\sum_{f_G}P(Y_i=c|H=\boldsymbol{h}_i,F_G=\boldsymbol{f}_G)\\
    &=\tau\frac{(\boldsymbol{w}_k^{(R)})^\top\boldsymbol{h}_i}{(\|\boldsymbol{w}_k^{(R)}\|+\gamma)\|\boldsymbol{h}_i\|}
\end{aligned}
    \end{equation}

Furthermore, we implement a deconfounding intervention on the causal path $H\rightarrow Y$. After training, $\boldsymbol{d}_G$ also becomes a constant, equal to $\boldsymbol{d}_G^{(R)}$ at the end of training. Therefore,

\begin{equation}
\begin{aligned}
    &\hat p_{i,c}^{infer}\\
    &=P\Big(Y_i=c|do(H=\boldsymbol{h}_i), F_k=\boldsymbol{f}_k, F_G=\boldsymbol{f}_G\Big)\\
    &\quad-P\Big(Y_i=c|do(H=\boldsymbol{d}_G), F_k=\boldsymbol{f}_k, F_G=\boldsymbol{f}_G\Big)\\
    &=\tau\left(\frac{(\boldsymbol{w}_k^{(R)})^\top\boldsymbol{h}_i}{(\|\boldsymbol{w}_k^{(R)}\|+\gamma)\|\boldsymbol{h}_i\|}\right.\\
    &\left.\quad-\alpha\cdot\frac{cos(\boldsymbol{h}_i,\boldsymbol{\hat{d}}_G^{(R)})\cdot(\boldsymbol{w}_k^{(R)})^\top\boldsymbol{\hat{d}}_G^{(R)}}{\|\boldsymbol{w}_k^{(R)}\|+\gamma}\right)
\end{aligned}
    \end{equation}

For the inference phase of the proposed CAFE algorithm, the detailed algorithmic description is outlined in Alg. 2. As indicated in the bottom right corner of Fig. 3, since the inference phase does not involve local optimizers from individual clients, the Feature Calibration module no longer includes the $\boldsymbol{d}_k$ component. Furthermore, as it is not influenced by local historical update directions, the inference phase of CAFE also omits the History-Aware Average module.

\section{Experiment}
\label{sec.6}
\subsection{Experimental Setup}
\begin{figure}
    \centering
    \includegraphics[width=1\linewidth]{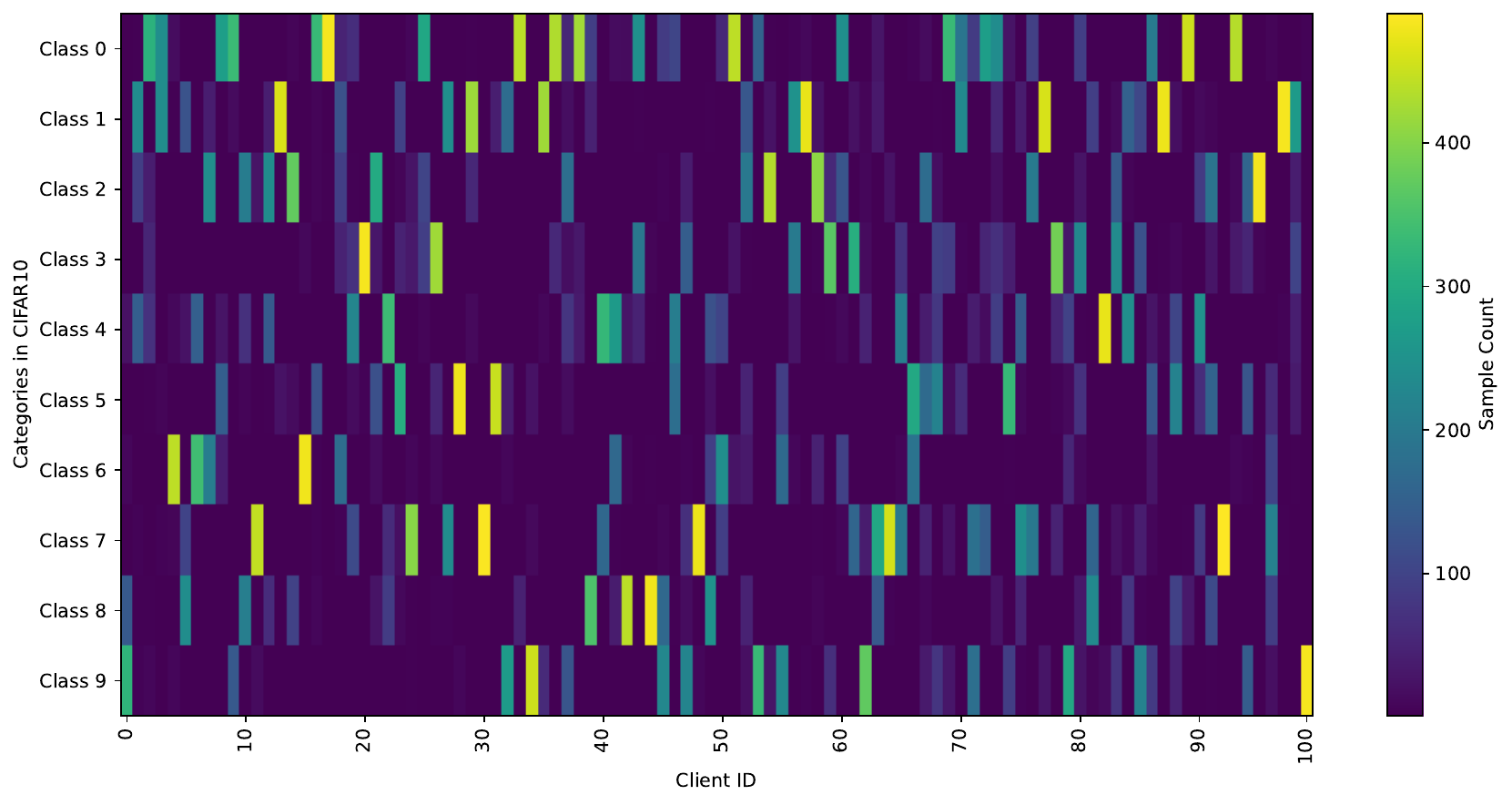}
    \caption{Enter Caption}
    \label{fig:enter-label}
\end{figure}
We evaluate our proposed \textbf{CAFE} method for federated learning classification tasks using the CIFAR-10, CIFAR-100, and Fashion-MNIST datasets. The data is split according to the original train-test ratios, and Dirichlet distributions (with parameters 0.1, 0.5, and 1.0) simulate varying degrees of data heterogeneity across clients. We use 100 clients in total, ensuring non-iid data distributions and significant class imbalance within each client. Figure 3 shows the CIFAR-10 data split result when Dir=0.1, where each color block represents the number of samples per class on a client. We also introduce a communication frequency parameter (CF) to simulate device heterogeneity in asynchronous FL, representing how often devices upload gradients based on their computational and communication capabilities, with CF values set at 0.1, 0.5, and 1.0.

Our experimental setup includes a single NVIDIA RTX 3090 GPU with 128GB memory. Client sampling rates are set at 10\%, 30\%, and 50\%, with a time threshold for waiting. The optimization algorithm used is momentum gradient descent (MGD) with a learning rate of 0.001. Each communication round includes 5 local training periods, with additional experiments conducted using 10 and 20 local training periods. The total number of communication rounds is 300. Baseline methods for comparison include FedAvg, FedProx, MOON, and FedProto. The evaluation metric is classification accuracy on client test sets, with all experiments repeated three times and average results reported.

\subsection{Benchmark}
To comprehensively evaluate the performance of the proposed CAFE, we adopted the following comparison methods:

\textbf{FedAvg}: As a typical solution introduced in \cite{ref16}, FedAvg selects a subset of clients in each communication round, initializes the client models with $\boldsymbol{w}$, updates the local models $\boldsymbol{w_i}$ by minimizing $\mathcal{L}_i(\boldsymbol{w})$, and aggregates the local models $\boldsymbol{w_i}$ into a new global model w until $\mathcal{L}(\boldsymbol{w})$ reaches a stationary point.

\textbf{FedProx}: FedProx \cite{ref13} introduces an additional Euclidean regularization term between the local and global models in the local optimization problem.

\textbf{MOON}: MOON \cite{ref53} leverages the feature similarity between the client model and the previous local model, using the global model as a contrastive regularizer to correct local training for each client.

\textbf{FedProto}: FedProto \cite{ref52} requires clients and the server to communicate via abstract class prototypes rather than gradients. It aggregates local prototypes collected from different clients and sends the global prototype back to all clients to regulate local model training.

\subsection{Experimental Results}
\begin{figure}
    \centering
    \includegraphics[width=0.8\linewidth]{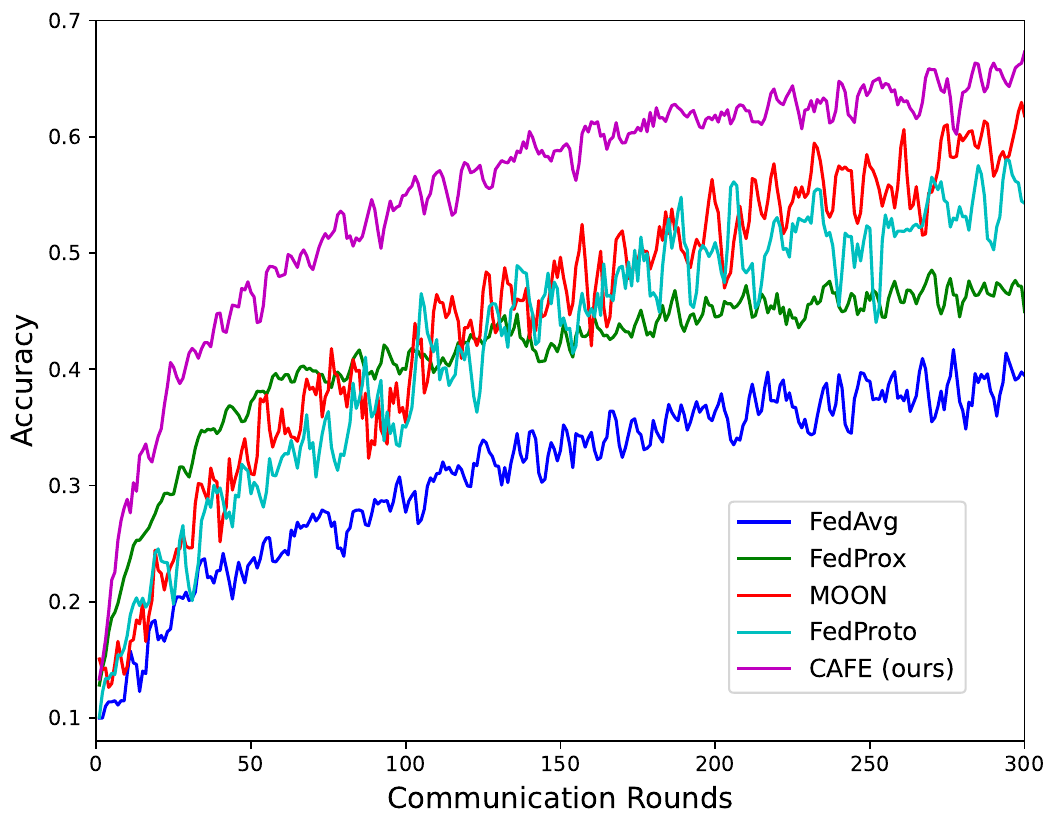}
    \caption{Accuracy}
    \label{fig:enter-label}
\end{figure}

\begin{figure}
    \centering
    \includegraphics[width=0.8\linewidth]{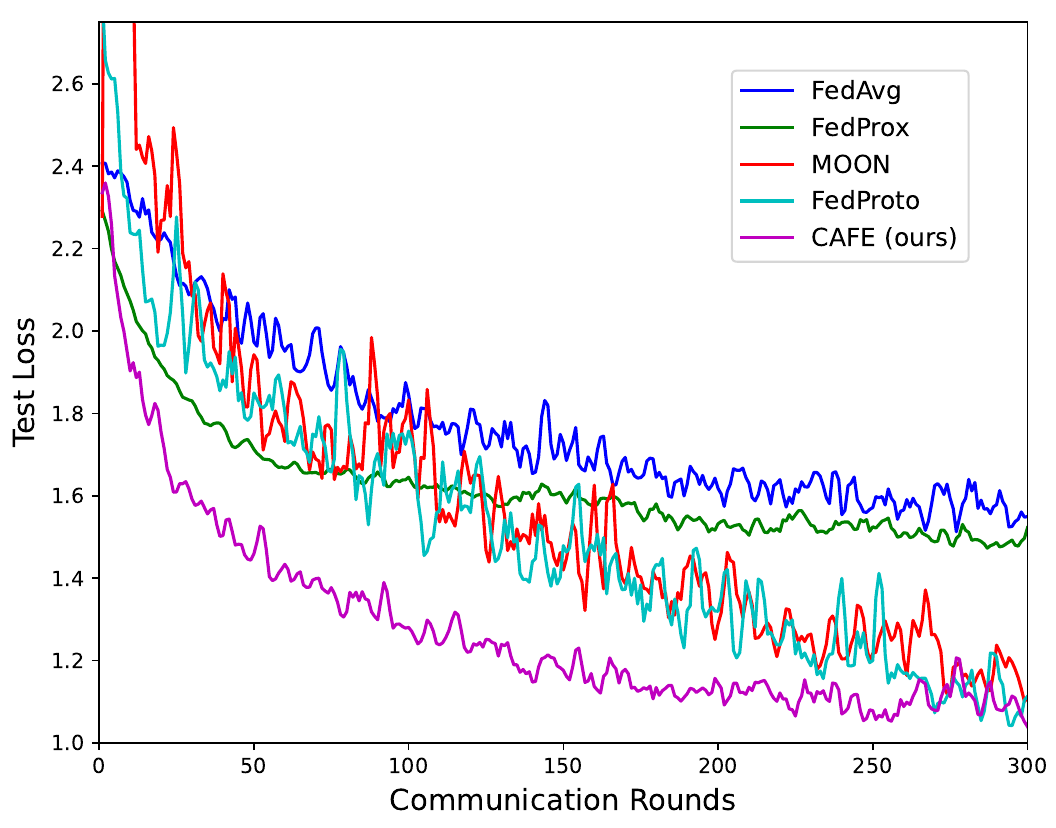}
    \caption{Test Loss}
    \label{fig:enter-label}
\end{figure}

Figure 5 presents the top-1 accuracy on the CIFAR-10-LT dataset under settings of CF=0.1, Dir=0.1, and a client sampling rate of 0.1, with local epochs set to 5 and communication rounds set to 200. Our proposed CAFE method achieves higher top-1 accuracy compared to baseline methods. Figure 6 illustrates the test loss descent trend under the same settings, demonstrating that CAFE also converges faster. We further evaluate the learning performance under various parameter settings.

\begin{table*}[h]
\centering
\caption{feature drift with Fixed Communication Frequency (CF = 0.1) and Different Class Imbalance Degrees (Dirichlet).}
\label{tab:drift_auc}
\begin{tabular}{lccc|ccc|ccc}
\toprule
\textbf{Method} & \multicolumn{3}{c|}{\textbf{FMNIST-LT}} & \multicolumn{3}{c|}{\textbf{CIFAR-10-LT}} & \multicolumn{3}{c}{\textbf{CIFAR-100-LT}} \\ 
\textbf{lr = 0.001} & Dir=0.1 & Dir=0.5 & Dir=1.0 & Dir=0.1 & Dir=0.5 & Dir=1.0 & Dir=0.1 & Dir=0.5 & Dir=1.0 \\ 
\midrule
FedAvg   & 87.22 & 87.89 & 89.91 & 41.71 & 43.31 & 45.59 & 36.00 & 36.52 & 37.13 \\
FedProx  & 88.51 & 88.94 & 90.47 & 48.53 & 51.57 & 56.66 & 36.71 & 37.36 & 38.22 \\
MOON     & 88.36 & \underline{90.40} & 91.35 & \underline{62.96} & \underline{64.66} & \underline{67.95} & 37.95 & 38.83 & \underline{39.93} \\
FedProto & \underline{88.98} & 90.33 & \underline{91.67} & 58.02 & 61.29 & 67.10 & \underline{38.12} & \underline{39.21} & 39.84 \\
CAFE (Ours) & \textbf{89.44} & \textbf{90.88} & \textbf{91.70} & \textbf{67.29} & \textbf{68.72} & \textbf{71.35} & \textbf{40.30} & \textbf{40.98} & \textbf{41.55} \\
\bottomrule
\end{tabular}
\end{table*}

\begin{table*}[h]
\centering
\caption{feature drift with fixed class imbalance degree Dirichlet (Dir=0.1) and different Communication Frequency (CF, linear ratio).}
\label{tab:drift_auc}
\begin{tabular}{lccc|ccc|ccc}
\toprule
\textbf{Method} & \multicolumn{3}{c|}{\textbf{FMNIST-LT}} & \multicolumn{3}{c|}{\textbf{CIFAR-10-LT}} & \multicolumn{3}{c}{\textbf{CIFAR-100-LT}} \\ 
\textbf{lr = 0.001} & CF=0.1 & CF=0.5 & CF=1.0 & CF=0.1 & CF=0.5 & CF=1.0 & CF=0.1 & CF=0.5 & CF=1.0 \\ 
\midrule
FedAvg   & 87.22 & 87.92 & 89.14 & 41.71 & 43.11 & 43.57 & 36.00 & 37.96 & 38.55 \\
FedProx  & 88.51 & 88.83 & 89.99 & 48.53 & 49.99 & 52.10 & 36.71 & 37.53 & 38.15 \\
MOON     & \underline{89.36} & 90.00 & 90.68 & \underline{62.96} & \underline{64.51} & \underline{67.32} & 37.95 & 38.67 & \underline{40.13} \\
FedProto & 89.41 & \underline{90.36} & \underline{90.51} & 58.02 & 60.66 & 64.45 & \underline{38.12} & \underline{38.99} & 39.22 \\
CAFE (Ours) & \textbf{89.44} & \textbf{90.47} & \textbf{90.92} & \textbf{67.29} & \textbf{68.81} & \textbf{69.64} & \textbf{40.30} & \textbf{40.90} & \textbf{42.11} \\
\bottomrule
\end{tabular}
\end{table*}

\begin{figure}
    \centering
    \includegraphics[width=1\linewidth]{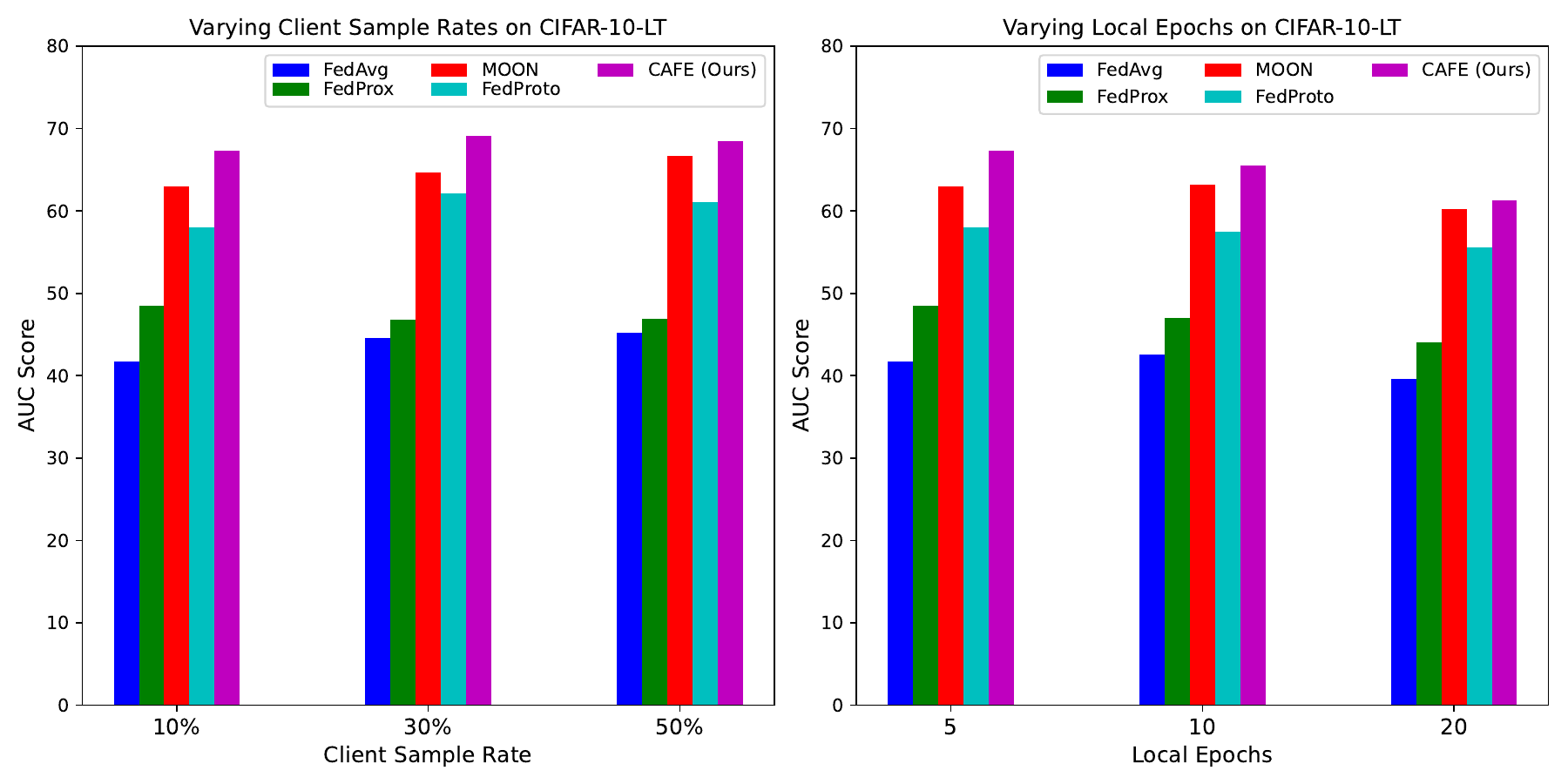}
    \caption{AUC results with varying client sample rates and local epochs on CIFAR-10-LT (lr=0.001).}
    \label{fig:enter-label}
\end{figure}

\subsubsection{Performance under class imbalance: }
Table I demonstrates that CAFE provides significant gains under different label skew settings. Regardless of the dataset, \(Dir = 1\) represents a data distribution closer to uniform, while \(Dir = 0.1\) indicates a more severe label distribution skew, with varying numbers of samples per client. The results show that CAFE consistently outperforms state-of-the-art methods across all datasets, with the performance gap widening as the complexity and heterogeneity of the datasets increase. For example, although the performance improvement of CAFE on the FMNIST-LT dataset is marginal, it is more pronounced on the CIFAR-10 dataset and becomes even more significant as the Dir value increases. This makes CAFE particularly suitable for scenarios with extreme data distributions. On the CIFAR-10-LT dataset, when \(Dir = 0.1\) and \(CF = 0.1\), the average test accuracy of CAFE is 67.29\%, compared to 62.96\% for the second-best model. This significant improvement is consistently observed across all cases, demonstrating the effectiveness of our proposed CAFE method. Moreover, for more challenging tasks, such as CIFAR-100-LT, the proposed method still achieved the best performance (approximately a 2\% accuracy improvement).

\subsubsection{Performance under participant imbalance: }
According to Table II, our method achieved higher accuracy than all baselines on the three datasets. Although the performance of all methods decreases as participant imbalance increases, the decline in CAFE is much smaller than that of other methods. For instance, when \(CF\) changes from 1.0 to 0.5, the top-1 accuracy of most baselines on CIFAR-10-LT drops by approximately 2\% to 3\%, which is about three times that of FedFA. When tested on the complex dataset CIFAR-100-LT, there is also a performance improvement of over 2\%.

\subsubsection{Performance under Different Client Sampling Rates: }
We further explore the impact of the federated setting. As shown in Table III, a larger client sampling rate does not always lead to better test accuracy for all methods, but there is an overall upward trend. Specifically, the FedProx method first decreases and then increases, while the FedProto method first increases and then decreases, similar to our proposed CAFE method. Overall, regardless of the client sampling ratio, CAFE consistently achieves the best performance.

\subsubsection{Performance under Different Local Epochs: }
As shown in Table IV, a larger number of local epochs negatively impacts performance, as larger local epochs make the update directions more scattered under extreme data heterogeneity. The CAFE method, which employs a causal classifier to learn invariant features during local training, continues to outperform other baselines.

\subsection{Ablation Studies}

\begin{figure}
    \centering
    \includegraphics[width=0.8\linewidth]{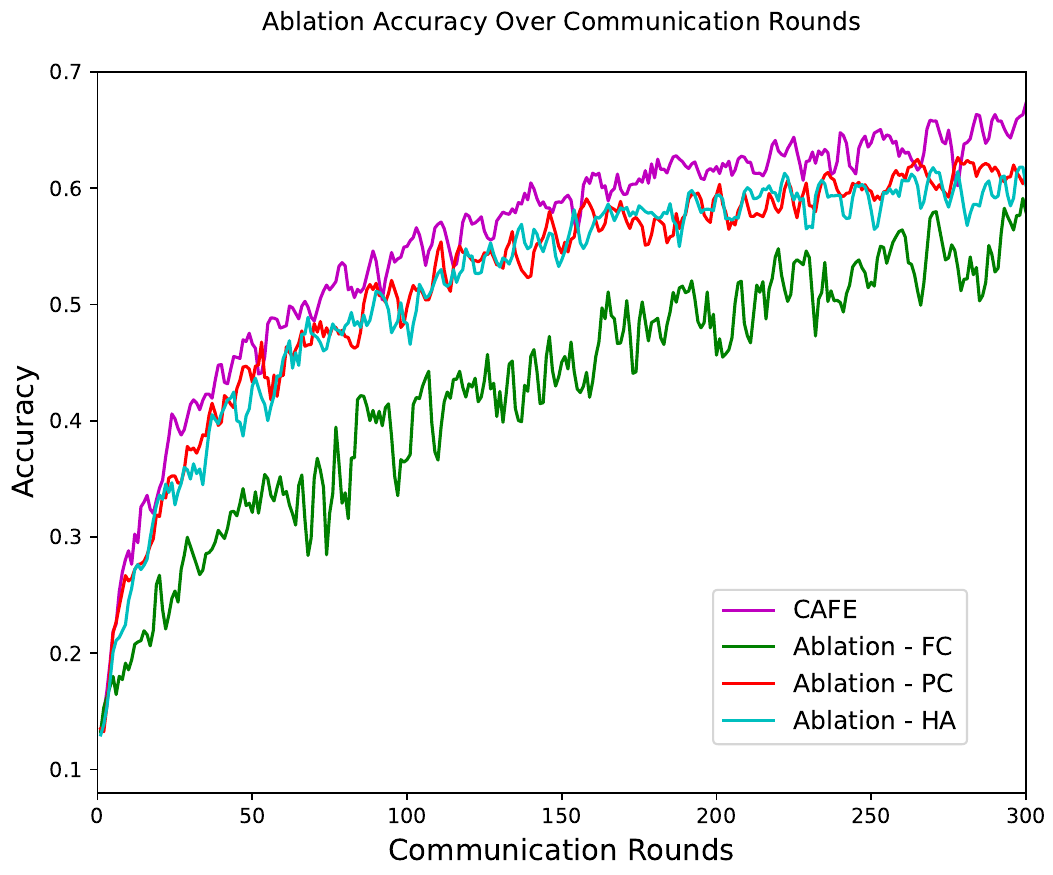}
    \caption{ablation accuracy}
    \label{fig:enter-label}
\end{figure}

As illustrated in Fig.6, we conducted ablation studies to investigate the impact of each module in the model. Besides the proposed CAFE algorithm, the other three curves represent the performance of CAFE without Parameter Calibration (PC), CAFE without Feature Calibration (FC), and CAFE without History-Aware Average (HA). The experimental results show that the model's performance significantly deteriorates when any of the PC, FC, or HA modules are removed, validating the effectiveness of the collaborative functioning of these three modules. Additionally, the results of the ablation studies align with the main idea of this paper, which posits that a complete causal intervention process can correct drifts and achieve more satisfactory performance.

Observations reveal that removing the Parameter Calibration or History-Aware Average module alone results in model performance comparable to the best baseline. Feature Calibration plays the most critical role in CAFE, as analyzed in Section 4, due to the significant feature drift present during the model training process.

\section{Conclusion}
\label{sec.7}
This paper is the first to reveal the drifts inherent in federated learning training mechanisms and identify them as harmful confounding factors from a causal perspective. Consequently, we propose a Causal Drift-Aware Federated Learning (CAFE) method aimed at eliminating drifts in embedding vectors caused by spurious correlations and indirect effects. Within the CAFE framework, we construct causal graphs for both the training and inference processes of federated learning, mitigating the interference of confounders and mediators to capture the invariant components within embedding vectors. Both theoretical analysis and experimental results demonstrate that the proposed CAFE method significantly outperforms existing approaches. The performance improvement is not only due to the resolution of specific imbalanced scenarios but also exhibits strong robustness across diverse environments. Moreover, our method determines drifts based on model parameters without relying on access to the class distributions of other clients, thus enabling operation under privacy-preserving conditions.

\bibliographystyle{IEEEtran}

\bibliography{reference}

\vspace{11pt}

\vfill

\end{document}